\documentclass[final,12pt]{colt2021} 

\title[Sparse distribution estimation under communication constraints]{Breaking The Dimension Dependence in Sparse Distribution Estimation under Communication Constraints}
\usepackage{times}
\usepackage{setspace}
\usepackage[utf8]{inputenc} 
\usepackage[T1]{fontenc}
\usepackage{url}
\usepackage{ifthen}
\usepackage{cite}
\usepackage{amsmath} 
\usepackage{graphicx}
\usepackage{bm}
\usepackage{bbm}
\usepackage{color}
\usepackage{mathtools}
\usepackage{amsmath,amssymb,mathrsfs}
\usepackage{thm-restate}
\usepackage{bm}
\usepackage{bbm}
\newtheorem{claim}{Claim}[section]

\newcommand\undermat[2]{%
	\makebox[0pt][l]{$\smash{\underbrace{\phantom{%
					\begin{matrix}#2\end{matrix}}}_{\text{$#1$}}}$}#2}

\newcommand{\lp}{\left(}
\newcommand{\rp}{\right)}
\newcommand{\lb}{\left[}
\newcommand{\rb}{\right]}
\newcommand{\lbp}{\left\{}
\newcommand{\rbp}{\right\}}
\newcommand{\lba}{\left\lvert}
\newcommand{\rba}{\right\rvert}
\newcommand{\lV}{\left\lVert}
\newcommand{\rV}{\right\rVert}
\newcommand{\mv}{\middle\vert}

\newcommand{\mcal}{\mathcal}

\newcommand{\bbm}{\mathbbm}
\newcommand{\mbb}{\mathbb}
\newcommand{\msf}{\mathsf}
\newcommand{\la}{\leftarrow}
\newcommand{\ra}{\rightarrow}

\newcommand{\eqDef}{\triangleq}
\newcommand{\diid}{\overset{\text{i.i.d.}}{\sim}}

\newcommand{\E}{\mathbb{E}}
\newcommand{\Var}{\mathsf{Var}}

\renewcommand{\Pr}{\mathbb{P}}

\usepackage{enumitem}

\coltauthor{%
 \Name{Wei-Ning Chen} \Email{wnchen@stanford.edu}\\
 \addr Department of Electrical Engineering, Stanford University
 \AND
 \Name{Peter Kairouz} \Email{kairouz@google.com}\\
 \addr Google Research%
 \AND
 \Name{Ayfer \"Ozg\"ur} \Email{aozgur@stanford.edu}\\
 \addr Department of Electrical Engineering, Stanford University
}

\begin{document}

\maketitle

\begin{abstract}%
We consider the problem of estimating a $d$-dimensional $s$-sparse discrete distribution from its samples observed under a $b$-bit communication constraint. The best-known previous result on $\ell_2$ estimation error for this problem is $O\lp \frac{s\log\lp {d}/{s}\rp}{n2^b}\rp$. Surprisingly, we show that when sample size $n$ exceeds a minimum threshold $n^*(s, d, b)$, we can achieve an $\ell_2$ estimation error of $O\lp \frac{s}{n2^b}\rp$. This implies that when $n>n^*(s, d, b)$ the convergence rate does not depend on the ambient dimension $d$ and is the same as knowing the support of the distribution beforehand. 

We next ask the question: ``what is the minimum $n^*(s, d, b)$ that allows dimension-free convergence?''.  To upper bound $n^*(s, d, b)$, we develop novel localization schemes to accurately and efficiently localize the unknown support. For the non-interactive setting, we show that $n^*(s, d, b) = O\lp \min \lp {d^2\log^2 d}/{2^b}, {s^4\log^2 d}/{2^b}\rp \rp$. Moreover,  we connect the problem with non-adaptive group testing and obtain a polynomial-time estimation scheme when $n = \tilde{\Omega}\lp{s^4\log^4 d}/{2^b}\rp$. This group testing based scheme is adaptive to the sparsity parameter $s$, and hence can be applied without knowing it. For the interactive setting, we propose a novel tree-based estimation scheme and show that the minimum sample-size needed to achieve dimension-free convergence can be further reduced to $n^*(s, d, b) = \tilde{O}\lp {s^2\log^2 d}/{2^b} \rp$.
\end{abstract}
\setcounter{tocdepth}{2}
\tableofcontents
\section{Introduction}
Estimating a distribution from its samples is a fundamental unsupervised learning problem that has been studied in statistics since the late nineteenth century \citep{pearson1894contributions}. Motivated by the fact that data is increasingly being generated ``at the edge'' \citep{kairouz2019advances} by countless sensors, smartphones, and other
devices, away from the central servers that churn through it, there has been significant recent interest in studying this problem in a distributed setting. Assume we observe $n$ i.i.d. samples from an unknown discrete distribution $p$, $ X_1, X_2,\dots, X_n \sim p,$
but each sample $X_i$ is observed at a different client $i$. Each client has a finite communication budget, say $b$ bits to communicate its sample to a central server which wants to estimate the unknown distribution $p$ under squared $\ell_2$ loss. Recent works \citep{han2018distributed, barnes2019lower, han2018geometric} showed that the estimation error in the distributed case can be as large as $O\lp \frac{d}{n2^b}\rp$ where $d$ is the domain size of the distribution. As compared to the classical (centralized) case where a simple empirical frequency estimator yields an $\ell_2$ error of $O(\frac{1}{n})$, this implies a significant penalty when $b$ is small and $d$ is large. Moreover in the classical case, the empirical frequency estimator can be applied without any knowledge of the domain of the distribution while the distributed scheme in \citep{han2018distributed} requires $d$ to be known ahead of time at all the nodes.

Fortunately, in many real-world applications such as location tracking, language modeling and web-browsing, the underlying distribution is often supported only on a sparse but unknown subset of size $s$, denoted $\msf{supp}(p)$, of the ambient domain of size $d$. Motivated by the utility gains due to sparsity in many high-dimensional statistical problems, we can ask whether the factor of $d$ penalty in distributed estimation can be avoided in settings where the underlying distribution is inherently sparse. For example, compressed sensing  \citep{donoho2006compressed, candes2006stable} suggests that a sparse vector of dimension $d$ with $s$ non-zero elements can be losslessly compressed  into a $s\log (d/s)$-dimensional vector. This may lead one to suspect that in the sparse case, the dimensionality penalty $d$  can be replaced by the ``effective" dimension of the problem $s\log (d/s)$, yielding an $\ell_2$ error  of $O\lp \frac{s\log (d/s)}{n2^b}\rp$ for distributed estimation (assuming that $b \leq \log s$). This result has been recently shown in \citep{acharya2020estimating}. Note that when the support of the sparse distribution is given beforehand, the earlier results of \citep{han2018distributed, barnes2019lower, han2018geometric} imply an $\ell_2$ error of $O\lp \frac{s}{n2^b}\rp$. However, the additional logarithmic dependence on $d$ in \citep{acharya2020estimating} appears natural and inevitable in light of classical results on sparsity. 

\paragraph{Our contributions} In this paper, we prove that one can surprisingly eliminate the $\log d $ term and achieve a \emph{dimension-free}\footnote{This means that the convergence rate does not depend on the ambient dimension $d$.} convergence rate $O\lp \frac{s}{n2^b} \rp$, as long as $n$ is larger than a threshold $n^*(s, d, b)$. To achieve $O\lp \frac{s}{n2^b} \rp$ convergence, we propose a two-stage scheme where in the first stage we use a subset of the samples  to \emph{localize} the support of $p$, and then use the remaining samples to refine the estimation. Our key contribution is to carefully design the localization stage, which allows us to improve the best-known result $O\lp\frac{s\log (d/s)}{n2^b} \rp$ to $O\lp \frac{s}{n2^b}\rp$. Note that $O\lp \frac{s}{n2^b}\rp$ is optimal since it is also the minimax convergence rate when $\msf{supp}(p)$ is given beforehand. To our knowledge, this is the first work that observes such dimension independent convergence in a sparse setting. 

A natural next step for the sparse distribution estimation task is to ask: \emph{``what is the minimum sample size $n^*(s,d,b)$ needed to achieve the dimension-free convergence rate $O\lp \frac{s}{n2^b}\rp$?''} We investigate this question in the non-interactive (each client encodes its observation independently) and interactive settings (the clients can interact in a sequential fashion). For the non-interactive setting, a simple grouping scheme for localization leads to an upper bound of $n^*(s, d, b) = O\lp {d^2\log^2 d}/{2^b} \rp$. On the other hand, by using carefully constructed random hash functions in the localization step, we show that $n^*(s, d, b) = O\lp {s^4\log^2 d}/{2^b} \rp$, which depends only logarithmically on $d$. However, the decoding algorithm for this scheme involves searching over all possible $s$-sparse supports and is computationally inefficient. To resolve the computational issue, we make a non-trivial connection to non-adaptive group testing, showing that as long as $n = \tilde{\Omega}\lp {s^4\log^4 d}/{2^b}\rp$\footnote{For ease of presentation, we use the notation $\tilde{O}(\cdot)$ and $\tilde{\Omega}(\cdot)$ to hide all dependence on $\log s$ and $\log\log d$.}, we can achieve the optimal estimation error in { $\msf{poly}(s, \log d)$} time. {  The resultant group testing based scheme adapts to the sparsity parameter $s$, and can be applied without knowing it.}
For the sequentially interactive case, we propose a tree-based scheme that achieves the optimal convergence rate when $n = \tilde{\Omega}\lp{ s^2\log^2 d}/{2^b}\rp$, showing that interaction between nodes can lead to smaller sample size. Lower bounds on $\ell_1$ sample complexity developed in  \citep{acharya2020estimating, acharya2019inference} imply that 1) $n^*(s, d, b) = \Omega\lp \frac{s^2\log(d/s)}{2^b} \rp$  in the non-interactive case; and $n^*(s, d, b) = \Omega\lp \frac{s^2}{2^b} \rp$ for the interactive setting. This implies that the requirement on $n$ in our sequentially interactive scheme is tight up to logarithmic terms, while there is a gap between upper and lower bounds on $n^*(s, d, b)$ (in terms of their dependence on $s$) in the non-interactive case.  See Table~\ref{tbl:1} for a comparison between different localization schemes.

\begin{table}[b!]
	\centering
	\begin{tabular}{|l|c|c|c|c|c|}
		\hline 
		$\vphantom{\Theta \lp 1 \rp}$ & $n^*\lp d, s, b\rp$ & Decoding time &Interactivity &Randomness \\ 
		\hline 
		A. uniform grouping & $\Omega\lp \frac{d^2\log^2 d}{2^b} \rp$ &$O\lp n 2^b \rp$ &non-interactive &public-coin \\
		\hline 
		B. random hashing& $\Omega\lp \frac{s^4\log^2\lp\frac{d}{s}\rp}{2^b} \rp$ &$O\lp n 2^b  d^s \rp$ &non-interactive &public-coin \\
		\hline
		C. group testing & $\tilde{\Omega}\lp \frac{s^4\log^4 d}{2^b} \rp$ &$O\lp n 2^b + \msf{poly}(s, \log d) \rp$ &non-interactive &public-coin \\
		\hline
		D. tree-based & $\tilde{\Omega}\lp \frac{s^2\log^2 d}{2^b} \rp$ &$\tilde{O}\lp n 2^b + s\log d \rp$ &interactive &private-coin \\
		\hline
	\end{tabular}
	\caption{Performance of using different localization schemes.} \label{tbl:1}
\end{table}

We also study several natural extensions of sparse distribution estimation. When the data is \emph{distribution-free} and the goal is to estimate the empirical frequency of the symbols held by different nodes, we show that our schemes extend naturally and we obtain the same convergence guarantees as in the distributional settings. Second, we show that our schemes extend to the case when $p$ is \emph{approximately-sparse}, and we provide an upper bound on the estimation error.

\paragraph{Notation and setup}
The general distributed statistical tasks we consider in this paper can be formulated as follows: each one of the $n$ clients has local data $X_i \sim p$ and sends a message $Y_i \in \mcal{Y}$ to the server, who upon receiving $Y^n$ aims to estimate the underlying distribution $p \in \mcal{P}_{s, d}$, where
$$ \mcal{P}_{s, d} \eqDef \lbp p = (p_1,...,p_d) \, \mv \,  p \in [0, 1]^d, \sum_j p_j =1, \lV p \rV_0 \leq s \rbp$$
is the collection of all $d$-dimensional $s$-sparse discrete distributions.

At client $i$, the message $Y_i$ is generated via some encoding channel (a randomized mapping that possibly uses shared randomness across participating clients and the server) denoted by a conditional probability $W_i(\cdot| X_i)$ (for the non-interactive setting) or $W_i(\cdot | X_i, Y^{i-1})$ (for the interactive setting).  The $b$-bit communication constraint restricts $\lba\mcal{Y}\rba \leq 2^b$, so without loss of generality we assume $\mcal{Y} = [2^b]$. When the context is clear, we sometimes view $W_i$ (in the non-interactive setting) as a $2^b \times d$ stochastic matrix, with $\lb W_i \rb_{y, x} \eqDef W_i(y|x)$. Finally, we call the tuple $\lp W^n, \hat{p}(Y^n)\rp$ an estimation scheme, where $\hat{p}\lp Y^n \rp$ is an estimator of $p$. 

Let $\Pi_{\msf{non-int}}$ be the collection of all non-interactive estimation schemes (i.e. $W_i$ is non-interactive for all $i \in [n]$) and $\Pi_{\msf{seq}}$ be the collection of sequentially interactive schemes. Our goal is to design a scheme $\lp W^n, \hat{p}\lp Y^n \rp\rp$ to minimize the $\ell_2$ (or $\ell_1$) estimation error:
\begin{align*}
	& r_{\msf{non-int}}\lp \ell_2, n, b\rp \eqDef \min_{\lp W^n, \hat{p}\rp \in \Pi_{\msf{non-int}}} \, \max_{p \in \mcal{P}_{s,d}} \E\lb \left\lVert p - \hat{p}\lp Y^n \rp \right\rVert^2_2 \rb,
\end{align*}
and $ r_{\msf{seq}}$ defined similarly with the minimum taken in $\Pi_{\msf{seq}}$. For $\ell_1$ error, we replace $\lV \cdot \rV^2_2$ in the above expectations with $\lV \cdot \rV_1$.


\paragraph{Related works}
Estimating discrete distributions is a fundamental task in statistical inference and has a rich literature \citep{barlow1972statistical,devroye1985nonparametric,devroye2012combinatorial,Silverman86}. In the case of communication constraints, the optimal convergence rate for discrete distribution estimation was established in  \citep{han2018geometric, han2018distributed, barnes2019lower, acharya2019inference,  acharya2019inference2, chen2020breaking} for the non-interactive setting, and \citep{barnes2019lower, acharya2020general} for general interactive models. 
The recent work of \citep{acharya2020estimating} considers the same problem under an $s$-sparsity assumption for the distribution. Our result improves on their result by removing the dimension-dependent $\log\lp\frac{d}{s}\rp$ term in their upper bound and hence matching a natural lower bound for the error (i.e. when the sparse support is known beforehand).

A slightly different but closely related problem is distributed heavy hitter detection and distribution estimation under local differential privacy constraints \citep{Bassily2015, Bassily2017, bun2019heavy, zhu2020federated} where no distributional assumption on the data is made. Although originally designed for preserving user privacy, the non-interactive tree-based schemes proposed in \citep{Bassily2017, bun2019heavy} can be modified to communication efficient (indeed $1$-bit) distribution estimation schemes. However, in the heavy-hitter problem, most of the results optimize with respect to $\ell_\infty$ error, and directly applying their frequency oracles leads to a sub-optimal convergence $O\lp \frac{s(\log d+ \log n)}{n}\rp$ in $\ell_2$ (see Section~\ref{sec:decoding} for a brief discussion). On the other hand, the interactive scheme in \citep{zhu2020federated}, which identifies (instead of estimating the frequencies of) heavy-hitter symbols, is similar to our proposed interactive scheme in nature. Nevertheless, we extend their result to a communication efficient scheme and explicitly characterize the convergence rate.

In our construction of a non-interactive polynomial-time decodable scheme (Section~\ref{sec:ncgt}), we map the support localization task into the non-adaptive combinatorial group testing problem \citep{dorfman1943detection, du2000combinatorial, ngo2000survey}. In particular, we show that the Kautz and Singleton's construction of test measurement matrices \citep{kautz1964nonrandom} can be used to design the local encoding channels and obtain a polynomial time decoding algorithm at the cost of a slightly larger sample size requirement. This novel connection opens the possibility to further harness the rich literature on group testing for building structured schemes for high-dimensional statistical problems with sparsity of the type we study here.

\paragraph{Organization} The rest of the paper is organized as follows: in Section~\ref{sec:result}, we present our main results, including the convergence rate and bounds on the minimum sample size requirement. In Section~\ref{sec:decoding}, we introduce the main idea of the generic two-stage scheme and propose a non-interactive construction for the second stage (i.e. the estimation phase). In Section~\ref{sec:smp}, we give three non-interactive localization schemes with different sample size requirements and decoding complexity. In Section~\ref{sec:interactive}, we introduce a tree-based interactive localization scheme and show how interactivity can be beneficial. Finally, we conclude our paper with a few non-trivial extensions and interesting open problems in Section~\ref{sec:conclusion}.

\section{Main Results}\label{sec:result}

Our main contribution is the design of both non-interactive and interactive schemes that achieve a dimension-free convergence rate $O\lp \frac{s}{n 2^b} \rp$ for the problem described in the earlier section. These schemes can generally be described by a two-stage protocol (see Algorithm~\ref{alg:two_stage}): the server uses the first $n_1$ clients to localize the support of p, denoted $\msf{supp}(p)$, and the remaining $n-n_1$ clients, in addition to the output of the first stage, to estimate $p$. We will later see that to estimate $\msf{supp}(p)$ accurately, there is a minimum requirement on $n$. We then propose different localization schemes that aim to minimize this requirement on $n$ in different parameter regimes. See Table~\ref{tbl:1} for a detailed comparison. 
Our first theorem upper bounds the requirement on $n$ under the non-interactive setting.

\begin{theorem}
	\label{thm:main_smp}
	\begin{enumerate}[topsep=0pt, leftmargin=2em]
		\item As long as
		$n = \Omega\lp \min\lp \frac{d^2\log^2 d}{\min(d, 2^b)}, \frac{s^4\log^2\lp\frac{d}{s}\rp}{ \min(2^b, s)}\rp\rp$, 
		 the $\ell_2$ and $\ell_1$ error for the non-interactive scheme is
		\begin{equation}\label{eq:l2_rate}
		\begin{cases}
		r_{\msf{non-int}}\lp \ell_2, n, b\rp = \Theta\lp\max\lp\frac{s}{n2^b}, \frac{1}{n}\rp\rp, \\
		r_{\msf{non-int}}\lp \ell_1, n, b\rp = \Theta\lp\max\lp\frac{s}{\sqrt{n2^b}}, \sqrt{\frac{s}{n}}\rp\rp.
		\end{cases}
		\end{equation}
		\item Moreover, if $n = \Omega\lp \frac{s^4\log^4 d\lp \log s +\log\log d \rp^2}{ \min(2^b, s)}  \rp$, then there exists a non-interactive scheme based on group testing that achieves the convergence rate \eqref{eq:l2_rate} with polynomial time decoding complexity
		$O\lp n2^b+\msf{poly}(s, \log d)\rp$. Further, the scheme is adaptive to $s$ (i.e. requires no knowledge of $s$).
	\end{enumerate}
\end{theorem}
Note that \eqref{eq:l2_rate} is the convergence rate \emph{with the knowledge} of $\msf{supp}(p)$ given beforehand, so the lower bound follows directly from standard distribution estimation results under communication constraints (by assuming $\msf{supp}(p)$ is known). See, for instance, \citep{barnes2019lower}. The achievability schemes (i.e. the upper bounds) are given in Section~\ref{sec:decoding} and Section~\ref{sec:smp}. Our results improve the best-known result \citep{acharya2020estimating}, which has convergence rate $ O\lp\frac{s\log (d/s)}{n2^b}\rp$, by a factor of $\log \lp{d}/{s}\rp$. 

We next show that with interactive schemes, we can further reduce the minimum sample-size requirement to $n = \tilde{\Omega}\lp {s^2\log d}/{2^b} \rp$. 

\begin{theorem}\label{thm:main_interactive}
	As long as $n = \Omega\lp \frac{s^2\log^2 d\lp \log s+\log\log d \rp^2}{ \min(2^b, s)} \rp$, the errors for sequentially interactive schemes $r_{\msf{seq}}\lp \ell_2, n, b\rp$ and $r_{\msf{seq}}\lp \ell_1, n, b\rp$ are the same as \eqref{eq:l2_rate}.
\end{theorem}
The lower bound on the convergence rate also follows from \citep{barnes2019lower}, and the upper bound is proved in Section~\ref{sec:interactive}. 

To translate Theorem~\ref{thm:main_smp} and Theorem~\ref{thm:main_interactive} into the language of sample complexity, let $\msf{SC_{non-int}}(\beta, d, s, b)$ be the $\ell_1$\footnote{Note that for discrete distributions, $\ell_1$ distance is the same as  total variation distance.} sample complexity of non-interactive setting, which is defined as 
$$ \msf{SC_{non-int}}(\beta, d, s, b) \eqDef \min\lbp n \in \mbb{N}\, \mv \, \min_{\lp W^n, \hat{p}\rp \in \Pi_{\msf{non-int}}} \, \max_{p \in \mcal{P}_{s,d}} \Pr\lbp \left\lVert p - \hat{p}\lp Y^n \rp \right\rVert_\msf{TV} \leq \beta \rbp \geq 0.9 \rbp, $$
and let $\msf{SC_{seq}}$ be defined in the similar way.
\begin{corollary}[Sample complexity]\label{cor:sc}
	 For $\beta \in (0,1)$, $\msf{SC_{non-int}}$ and $\msf{SC_{seq}}$ satisfy
	\begin{align*}
		& \msf{SC_{non-int}}(\beta, d, s, b) = O\lp \max\lp\frac{s^2}{\beta^2\min\lp 2^b, s \rp}, \min\lp \frac{d^2\log^2 d}{\min( 2^{b}, d)}, \frac{s^4\log^2\lp\frac{d}{s}\rp}{\min(2^b, s)} \rp \rp\rp, \text{ and }\\
		& \msf{SC_{seq}}(\beta, d, s, b) = O\lp \max\lp \frac{s^2}{\beta^2\min\lp 2^b, s \rp}, \frac{s^2\log^2 d\lp \log s+\log\log d \rp^2}{\min(2^b, s)} \rp\rp.
	\end{align*}
	{ 
	Moreover, it holds that
	\begin{align*}
		& \msf{SC_{non-int}}(\beta, d, s, b) = \Omega\lp \max\lp \frac{s^2}{\beta^2\min\lp 2^b, s \rp}, \frac{s^2 \log\lp d/s \rp}{\beta 2^b} \rp\rp, \text{ and }\\
		& \msf{SC_{seq}}(\beta, d, s, b) = \Omega \lp \frac{s^2}{\beta^2\min\lp 2^b, s \rp}\rp.
	\end{align*}
	}
\end{corollary}
{  The lower bounds in Corollary~\ref{cor:sc} are from \citep{acharya2020estimating, acharya2019inference}. Corollary~\ref{cor:sc} directly implies the following two facts. First, comparing the lower bounds with the upper bounds, we see that our achievability results are tight when 1) $0 < \beta \leq \frac{1}{s\log(d/s)}$ for the non-interactive setting, and 2) $0 < \beta \leq \frac{1}{\log d(\log s + \log\log d)}$ for the sequentially interactive setting. This suggests that for small enough $\beta$, increasing the ambient dimension $d$ does not increase the sample complexity at all. Second, we observe from Corollary~\ref{cor:sc} that the $\ell_1$ estimation error must be at least $\Theta(1)$ when 1) $n = O\lp \frac{s^2\log(d/s)}{ \min(2^b, s)} \rp$ for the non-interactive model and 2)  $n = O\lp \frac{s^2}{\min(2^b,s)} \rp$ for the sequentially interactive model. This results in the following lower bounds for the minimum sample size requirement: 1) $n^*(s, d, b) = \Omega\lp \frac{s^2\log(d/s)}{\min(2^b,s)} \rp$  in the non-interactive case and 2) $n^*(s, d, b) = \Omega\lp \frac{s^2}{ \min(2^b, s)} \rp$ for the interactive setting. This implies that the requirement on $n$ in our sequentially interactive scheme is tight up to logarithmic terms, while there is a gap between upper and lower bounds on $n^*(s, d, b)$ (in terms of their dependence on $s$) in the non-interactive case. We believe that the lower bound on $n^*(s, d, b)$ is also tight in the non-interactive case and the non-interactive schemes can be further improved.} 

Next, we show that our results extend naturally to the distribution-free setting where we do not assume any underlying distribution on $X^n$, and the goal is to estimate the empirical distribution $\pi \eqDef \lp \pi_1,...,\pi_d  \rp$ of $n$ local observations $X_1,...,X_n \in \lp j_1,...,j_s \rp\eqDef \mcal{J} \subset [d]$. Formally, let $\pi_j \eqDef \frac{1}{n}\sum_{i=1}^n\bbm{1}_{\lbp X_i = j\rbp}$, and define
\begin{align*}
	& \tilde{r}_{\msf{non-int}}\lp \ell_2, n, b\rp \eqDef \min_{\lp W^n, \hat{\pi}\rp \in \Pi_{\msf{non-int}}} \, \max_{X^n \in \mcal{J}^n} \E\lb \left\lVert \pi - \hat{\pi}\lp Y^n \rp \right\rVert^2_2 \rb
\end{align*}
and $\tilde{r}_{\msf{seq}}$ similarly. The expectation is over the (possibly shared) randomness used in the algorithm.
\begin{theorem}\label{thm:freq_est}
	The convergence rates as well as the sample-size requirements given in Theorem~\ref{thm:main_smp} and Theorem~\ref{thm:main_interactive} hold for 
	$\tilde{r}_{\msf{non-int}}$ and $\tilde{r}_{\msf{seq}}$.
\end{theorem}

Finally, when the target distribution $p$ is no longer $s$-sparse, we have the following bound on the $\ell_2$ estimation error:
\begin{theorem}\label{thm:almost_sparse}
	There exists a non-interactive scheme such that as long as $n = \Omega\lp {s^4\log^2\lp\frac{d}{s}\rp}/{\min(2^b,s)} \rp$, 
	$$ \E\lb \sum_{j\in[d]} \lp \hat{p}_j -p\rp^2 \rb  \leq C_2\cdot \lp\max\lp \frac{s}{n2^b}, \frac{1}{n}\rp\rp+C_3\cdot\sqrt{n\log n}\cdot 2^b\cdot\lp 1-P_\mcal{S} \rp, $$
	where $P_\mcal{S} \eqDef \sum_{j=1}^s p_{(j)}$ and $p_{(j)}$ is the $j$-th largest value of $p$.
\end{theorem}
\begin{remark}
	Note that if $p$ is exactly $s$-sparse, then $P_\mcal{S} = 1$, and Corollary~\ref{thm:almost_sparse} recovers the convergence rate and the second bound on the sample-size requirement in Corollary~\ref{thm:main_smp}. Moreover, if $1-P_\mcal{S} < \frac{\min\lp 1, s/2^b\rp}{n^{\frac{3}{2}}\log^{\frac{1}{2}}n 2^b}$, then we recover the convergence rate with exact $s$-sparsity.
\end{remark}
The proofs of Theorem~\ref{thm:freq_est} and Theorem~\ref{thm:almost_sparse} can be found in Section~\ref{sec:proofs}.


\section{A Two-stage Decoding Algorithm}\label{sec:decoding}
All of our schemes, both interactive and non-interactive, can be generally described by a two-step protocol (Algorithm~\ref{alg:two_stage}): we partition all participating clients into two groups, such that
\begin{itemize}[topsep=0pt, leftmargin=1em]
\setlength\itemsep{-0.1em}
	\item the first group of $n_1 \leq n/2$ clients are used to \emph{localize} the support of $p$, i.e. estimate  $\msf{supp}(p)$. In particular, let $\mcal{J}_\alpha \eqDef \lbp j \in [d] \, \mv \, p_j \geq \alpha \rbp$ be the collection of high probability symbols. Then from the  reports $Y^{n_1}$ of the first group of clients, the server generates an estimate on $\mcal{J}_\alpha$, such that 1) $\mcal{J}_\alpha \subseteq\hat{\mcal{J}}_\alpha$ with high probability; and 2) $ |\hat{\mcal{J}}_\alpha| \leq s$ almost surely. 
	\item the reports $Y^{n_2}$ of the second group of $n_2 \eqDef n - n_1$ clients are used to estimate $p_j$ with the knowledge of $\mcal{J}_\alpha$ from the first stage.
\end{itemize} 
Notice that although the decoding algorithm is two-stage, encoding at each client can be done simultaneously and does not require the knowledge of $\hat{\mathcal{J}}_\alpha$. Also note that $n_1, n_2$ and $\alpha$ are design parameters that will be specified later. For ease of presentation, we call the first phase \emph{localization} and the second phase \emph{estimation}. 

\begin{figure}
\begin{minipage}[t]{.48\linewidth}
\begin{algorithm2e}[H]
	\SetAlgoLined
	\KwIn{$Y^n = \lp Y^{n_1}, Y^{n_2}\rp$, $\alpha > 0$ }
	\KwOut{$\hat{p} = \lp \hat{p}_1, ..., \hat{p}_d\rp$}
	$\hat{J}_\alpha \la \msf{support\_localization}\lp Y^{n_1}, \alpha \rp$\;
	$\hat{p} \la \msf{estimation}\lp Y^{n_2}, \hat{J}_\alpha \rp$\;
	\KwRet{$\hat{p}$}
	\caption{Two-stage decoding alg}\label{alg:two_stage}
\end{algorithm2e}
\end{minipage}
\hfill
\begin{minipage}[t]{.48\linewidth}
\begin{algorithm2e}[H]
	\SetAlgoLined
		\setstretch{1.1}
	\KwIn{$X_i \in [d]$, $b \in \mbb{N}$ }
	\KwOut{$Y_i$}
	Generate random hash function $h_i(\cdot): [d] \ra [2^b]$ 
	by shared randomness\;
	\KwRet{$Y_i = h_i(X_i)$}
	\caption{Encoder in estimation stage}\label{alg:enc_estimation}
\end{algorithm2e}
\end{minipage}
\end{figure}

For the estimation phase, we adopt a scheme similar to \citep{acharya2020estimating}, where each client encodes their local observation via an independent $b$-bit random hash function\footnote{A randomized mapping $[d] \ra 2^b$ is called random hash function if $\forall x\in[d],\, y\in[2^b],\,\, \Pr\lbp h(x) = y \rbp = \frac{1}{2^b}$.} $h_i:[d]\ra[2^b]$ as described in Algorithm~\ref{alg:enc_estimation}. Upon receiving the hashed values from $n_2$ clients in the second stage, the server estimates the empirical frequencies of all symbols $j\in\hat{\mcal{J}}_\alpha$ by counting the number of clients $i\in[n_1+1:n]$ such that $Y_i = h_i(j)$, and sets $\hat{p}_j = 0$ for all $j \not\in \hat{\mcal{J}}_\alpha$:
\begin{equation*}
\hat{p}_j\lp Y^{n_2} \rp = 
\begin{cases}
\frac{\lp 2^b-1\rp\cdot \lp \sum_{i=n_1+1}^n \bbm{1}_{\lbp h_i(j) = Y_i \rbp} \rp}{n_2\cdot 2^b}-\frac{1}{2^b}, \, &\text{if } j \in \hat{\mcal{J}}_\alpha\\
0,\, & \text{else}.
\end{cases}
\end{equation*}
We provide a more formal description of the scheme in Section~\ref{sec:estimation_stage} of the appendix. The $\ell_2$ estimation error of this scheme can be controlled by the following lemma:
\begin{lemma}\label{cor:two_stage_error}
	The estimation error of the two-stage scheme is upper bounded by
	\begin{equation}\label{eq:two_stage_error}
		\E\lb \sum_{j\in[d]} \lp \hat{p}_j -p_j\rp^2 \rb \leq 2\Pr\lbp \mcal{J}_\alpha \not\subset \hat{\mcal{J}}_\alpha \rbp + s\alpha^2+\frac{s}{n_22^b}+\frac{1}{n_2}.
	\end{equation}
\end{lemma}
The first term in the left side of \eqref{eq:two_stage_error} corresponds to the probability of failure in the first stage, and the remaining terms correspond to the the $\ell_2$ error resulting from the second stage provided that the localization was done correctly in the first stage.


We will later see that we can make the failure probability decay exponentially fast with $n_1\alpha$ (recall that $n_1$ is the number of clients participating in the localization phase), that is,
$ \Pr\lbp \mcal{J}_\alpha \not\subset \hat{\mcal{J}}_\alpha\rbp = \exp\lp -\frac{n_1\alpha}{f\lp s, d \rp}\rp, $
where $f(s, d)$ depends only on $s$ and $d$. Therefore, if we pick $\alpha = \frac{1}{\sqrt{n 2^b}}$ and $n_1 = \Omega\lp f(s, d)\sqrt{n 2^b}\log n \rp$, we arrive at the dimension-free convergence in Theorem~\ref{thm:main_smp}. 

However, the condition that $n_1 \leq n$ imposes a minimum requirement on the amount of participating clients in the first stage, i.e. $n = \Omega\lp f^2(s,d)  \log^2f(s, d) \rp$. Hence it is critical to carefully design the localization scheme to get the best scaling for $f(s, d)$ with $s$ and $d$. 
For instance, if we use count-sketch based methods as in the heavy-hitter literature \citep{Bassily2017, bun2019heavy}, we can only localize to the resolution level $\alpha = O\lp \sqrt{(\log n+\log d)/n} \rp$\footnote{See, for instance Table~1 in \citep{bun2019heavy}, where we pick $\varepsilon = \Theta(1)$, $\beta = \Theta(1/n)$ and $|X| = d$ to translate to our settings.}, which is not enough to get the $O\lp \frac{s}{n2^b} \rp$ rate. In the next two sections, we introduce three non-interactive and one interactive localization schemes, which have different minimum sample-size requirements and decoding complexities, as summarized in Table~\ref{tbl:1}.

\section{Non-interactive Localization Schemes}\label{sec:smp}
In this section, we present three non-interactive schemes for the localization stage, each offering a different trade-off between the minimum number of samples to achieve the dimension-free convergence rate and decoding time. 
\subsection{Localization scheme~A: uniform grouping} 
Under the $b$-bit communication constraint, our first encoding scheme is based on the grouping idea \citep{han2018distributed}, where each client only encodes symbols in a pre-specified subset of $[d]$ and ignores others. 
In particular, we partition the $d$ symbols and $n_1$ clients into $M$ equal-sized groups (disjoint subsets) denoted by $\mcal{B}_1,...,\mcal{B}_M$ and $\mcal{G}_1,...,\mcal{G}_M$, respectively. Clients $i \in \mcal{G}_m$ are assigned to the subset of symbols $\mcal{B}_{m}$. This means that they only encode symbols in $\mcal{B}_m$ (and ignore their sample if it is not in $\mcal{B}_m$ and set $Y_i=0$). We set $M\eqDef d/(2^b-1)$ so that each $\mcal{B}_m$ contains exactly $2^b-1$ symbols, and thus the encoded message can be described in $b$ bits. Upon observing all messages from the clients, the server computes $\hat{\mcal{J}}_\alpha$ that contains all  symbols successfully signaled to it. Note that a symbol $j \in \mcal{J}_\alpha \cap \mcal{B}_{m}$ will be in $\hat{\mcal{J}}_\alpha$ if a client $i \in \mcal{G}_m$ observes $j$. 

\paragraph{Encoding}
Formally, assume $b \leq \log d$ and let $M\eqDef d/(2^b-1)$ and
$ \mcal{G}_m = \lbp i\in[n_1] \mv \, i \equiv m \pmod M \rbp. $
For client $i \in \mcal{G}_m$, they only reports information about $j \in \mcal{B}_m \eqDef \lb (m-1)\cdot (2^b-1)+1: m\cdot (2^b-1)\rb$.
Equivalently, for clients in $\mcal{G}_m$, their encoding channel matrices are 

\begin{equation}
W_i = \lb
\begin{array}{rrr|rrrr|rrr}
e_{2^b} &\cdots  &e_{2^b} \undermat{(m-1)(2^b-1)+1: m(2^b-1)}{  &e_{1} &e_{2} &\cdots &e_{2^b-1} \, }  & \, e_{2^b} &\cdots &e_{2^b}
\end{array}
\rb,
\end{equation} 
\vspace{1em}

\noindent where $e_{\ell} \in \lbp 0,1 \rbp^{2^b\times 1}$ is the $\ell$-th coordinate vector. 

\begin{figure}
\begin{minipage}[t]{.48\linewidth}
\begin{algorithm2e}[H]
	\SetAlgoLined
	\KwIn{$X_i \in [d]$, $b\in \mbb{N}$ }
	\KwOut{$Y_i$}
	Compute $M = d/\lp2^b-1\rp$\;
	$m\la i \msf{\,mod\,} M$\;
	\eIf{ {\small $X_i \in \lb (m-1)(2^b-1)+1: m(2^b-1)\rb$}}{ $Y_i \la X_i \msf{\,mod\,} 2^b$}{$Y_i \la 0$}
	\KwRet{$Y_i$}
	\caption{{\small Uniform grouping: encoding }}\label{alg:enc_localization_1}
\end{algorithm2e}
\end{minipage}
\hfill
\begin{minipage}[t]{.48\linewidth}
\begin{algorithm2e}[H]
	\SetAlgoLined
	\KwIn{{\small $Y^{n_1} \in [2^b]$}}
	\KwOut{{\small $\hat{\mcal{J}}_\alpha$}}
	Initialize {\small $\hat{\mcal{J}}_\alpha = \emptyset$}\;
	Compute $M = d/\lp2^b-1\rp$\;
	\For{ $i \in [n_1]$}{
		$m\la i \msf{\,mod\,} M$\;
		\If{$Y_i \neq 0$}{{\small Add $\hat{X}\eqDef m\lp 2^b -1\rp+Y_i$ into $\hat{\mcal{J}}_\alpha$}\;}
	}
	\KwRet{{\small $\hat{\mcal{J}}_\alpha$}}
	\caption{{\small Uniform grouping: decoding}}\label{alg:dec_localization_1}
\end{algorithm2e}
\end{minipage}
\end{figure}

\paragraph{Decoding}
Due to our construction of encoding functions, we see that as long as $i\in\mcal{G}_m$, 
$$ Y_i \neq 0 \Longleftrightarrow X_i \in \mcal{B}_m,$$
so the server can specify $X_i$ upon observing $Y_i$ by computing $X_i = (m-1) (2^b-1) + Y_i$. Therefore, defining 
$$ \hat{X}\lp Y_i \rp \eqDef \begin{cases}
X_i, \text{ if } X_i \in \mcal{B}_m\\
\texttt{null}, \text{ otherwise},
\end{cases} $$
then we can estimate $\mcal{J}_\alpha$ by $\hat{\mcal{J}}_\alpha \eqDef \lbp\hat{X}\lp Y_i\rp \mv i \in [n_1] \rbp$, that is, $\hat{\mcal{J}}_\alpha$ is the collection of all observed and successfully decoded symbols from the first $n_1$ clients. The details of the encoding and decoding algorithms are given in Algorithm~\ref{alg:enc_localization_1} and Algorithm~\ref{alg:dec_localization_1}.

\begin{lemma}\label{clm:1}
	Under the above encoding and decoding schemes (see Algorithm~\ref{alg:enc_localization_1}, Algorithm~\ref{alg:dec_localization_1} for the formal descriptions), we have
	$ \Pr\lbp \hat{\mcal{J}}_\alpha\not\subset \mcal{J}_\alpha \rbp \leq  s \exp\lp{-{n_1(2^b-1)\alpha}/{d}}\rp.$
\end{lemma}
The formal proof of Lemma~\ref{clm:1} is left to Section~\ref{sec:proofs}.
Finally, taking $\alpha = \frac{1}{\sqrt{n 2^b}}$ and combining Lemma~\ref{cor:two_stage_error} and Lemma~\ref{clm:1}, we arrive at the following bound for $r_{\msf{non-int}}(\ell_2,n,b)$:
\begin{align}\label{eq:l2_final_bdd}
    \E\lb \sum_{j\in[d]} \lp \hat{p}_j -p_j\rp^2 \rb 
    & \leq 2s e^{ -\frac{n_1\cdot (2^b-1)}{d\cdot \sqrt{n 2^b}}}+\frac{2s}{(n-n_1)\cdot 2^b}+\frac{1}{n-n_1}
     \leq C\cdot\lp\frac{s}{n 2^b}+ \frac{1}{n}\rp, 
\end{align}
where in the last inequality we choose $n_1 = n/2$ and assume $n \succeq \frac{ d^2\log^2 d}{2^b}$. This gives the first sample-size requirement in Theorem~\ref{thm:main_smp}. To bound $r_{\msf{non-int}}(\ell_1,n,b)$, we apply Jensen's and Cauchy-Schwarz inequalities to obtain
\begin{align}\label{eq:l_1_2}
	r_{\msf{non-int}}(\ell_1,n,b) = \E\lb \lV \hat{p} -p\rV_1 \rb \leq \sqrt{\E\lb \lV \hat{p} -p\rV^2_1 \rb} \leq \sqrt{2s\cdot\E\lb \lV \hat{p} -p\rV^2_2 \rb}.
\end{align}
The last inequality holds since by our construction, $\lba \msf{supp}\lp \hat{p}\rp \cup \msf{supp}\lp p \rp\rba < 2s$.

\subsection{Localization Scheme~B: non-uniform random hashing}\label{sec:localization_hasing}
Though Scheme~A achieves the minimax estimation error when $p \in \mcal{P}_d$ (see \citep{han2018distributed, han2018geometric}), it is indeed inefficient under the $s$-sparse assumption. This is because only a small fraction of symbols can be observed with non-zero probability, so for clients assigned to blocks that did not contain these symbols, they always encode their observations to $Y_i = 0$. In our second localization scheme, we aim to improve the encoding efficiency by using random hash functions. However, unlike in the estimation stage described in Section~\ref{sec:estimation_stage}, the random hash functions we use for localization are generated non-uniformly. 

\paragraph{Encoding}
Assume $b \leq \log s$. For $i\in[n_1]$, client $i$ generates their local random hash function $W_i(y|x)$ as follows: each column of $W_i$, denoted $W_i(\cdot|x) \in \lbp 0,1 \rbp^{2^b}$, is defined as the one-hot representation of $L_{i,x}$, where $L_{i,x} \in [2^b]$ follows a multinomial distribution
	\begin{equation}\label{eq:Lx}
		L_{i, x}\diid \msf{Mult}\lp1, \lp\frac{1}{s}, \frac{1}{s}, ..., \frac{1}{s}, 1-\frac{2^b-1}{s} \rp\rp.
	\end{equation}
Formally, we can be express
\begin{equation}\label{eq:channel_B}
	W_i(\cdot|x) \eqDef e_{L_{i, x}} \in \{0,1\}^{2^b\times 1}
\end{equation} for all $x \in [d]$ (where 
$e_{L}$ is the $L$-th standard basis vector). Since the $W_i$ corresponds to a deterministic mapping  (but randomly generated as described above), sometimes we write $Y_i = h_i(X_i)$ for simplicity. 
The encoding algorithm resembles the one in the estimation stage (Algorithm~\ref{alg:enc_estimation}), except that now the random hash functions are generated according to \eqref{eq:channel_B}.

\paragraph{Decoding}
The decoding rule is based on exhaustive search. Due to the $s$-sparse assumption, there are at most $N \eqDef {d \choose s}$ possibilities for $\msf{supp}(p)$, which we index by $\mcal{C}_1,...,\mcal{C}_N$. Hence the localization step can be cast into a multiple hypothesis testing problem: let $H_\ell$ be the hypothesis such that $\mcal{J}_\alpha \subseteq \mcal{C}_\ell$, for $\ell \in [N]$. To proceed, we first define the notion of consistency.
\begin{definition}\label{def:consistency}
	We say $H_\ell$ (or $\mcal{C}_\ell$) is \emph{consistent} with observations 
	$(W^n, Y^n)$ if $ \Pr\lbp Y^n \mv H_\ell, W^n \rbp > 0. $
\end{definition}
The decoding rule is as follows: upon observing local encoding functions and reports $\{(W_i, Y_i),$ $ i = 1,...,n_1\}$ from all clients, the server searches for all candidates $\mcal{C}_1,...,\mcal{C}_N$ and randomly picks one which is consistent with $(W^n, Y^n)$ as our estimate of $\mcal{J}_\alpha$. 

\begin{algorithm2e}[H]
	\SetAlgoLined
	\KwIn{$Y^{n_1} \in [2^b]$}
	\KwOut{$\hat{\mcal{J}}_\alpha$}
	Initialize $\texttt{consist} \la \msf{True}$\;
	Let $\mcal{C}_1,...,\mcal{C}_N$ be an enumerate of all $N\eqDef {d \choose s}$  size-$s$ subsets of $[d]$\;
	\For{$\ell \in [N]$}{
		$\texttt{consist} \la \bigwedge_{i\in[n_1]}\bigvee_{j\in\mcal{C}_\ell}\bbm{1}_{\lbp Y_i = h_i(j)\rbp}$\tcp*{check consistency}
		\If{$\texttt{consist}$}{
			$\hat{\mcal{J}}_\alpha \la \mcal{C}_\ell$\;
			break\; 
		}
	}
	\KwRet{$\hat{\mcal{J}}_\alpha$}
	\caption{Localization via random hash: decoding}\label{alg:dec_localization_2}
\end{algorithm2e}

By using non-uniform hash functions (i.e. generating $W_i(\cdot |x)$ according to \eqref{eq:channel_B}), we can improve the distinguishability of $W_i$ (which is formally defined in Definition~\ref{def:distinguish}), which reduces the probability of accepting false hypothesis $H_\ell$ for some $\mcal{J}_\alpha \not\subset\mcal{C}_\ell$. Indeed, we can obtain the following bound on $\Pr\lbp \mcal{J}_\alpha \not\subset \hat{\mcal{J}}_\alpha \rbp$:

\begin{lemma}\label{lemma:2}
	Under the above encoding and decoding rules,
	\begin{equation*}
		\Pr\lbp \mcal{J}_\alpha \not\subset \hat{\mcal{J}}_\alpha \rbp \leq \exp\lp -n_1  \alpha{(2^b-1)}/{4s} + C_0s\log\lp d/s\rp\rp,
	\end{equation*}
	for some universal constant $C_0 > 0$. Moreover, with probability $1$, $|\hat{\mcal{J}}_\alpha| = s$.
\end{lemma}

Again, we pick $n_1 = \frac{n}{2}$ and $\alpha = \frac{1}{\sqrt{n 2^b}}$. By combining Lemma~\ref{cor:two_stage_error} and Lemma~\ref{lemma:2}, we arrive at
\begin{align}\label{eq:hash_error_1}
\E\lb \sum_{j\in[d]} \lp \hat{p}_j -p\rp^2 \rb 
& \leq 2\exp\lp -\sqrt{\frac{n\lp 2^b -1 \rp^2}{2^b}}\cdot\frac{1}{8s}+C_0s\log\lp\frac{d}{s}\rp \rp+\frac{3s}{n2^b}+\frac{2}{n}.
\end{align}
To ensure the first term less than $O({1}/{n})$, we introduce the following simple but useful lemma.
\begin{lemma}\label{lemma:n_requirement}
	Let $f_1(s, d, b) \geq 300$, $f_2(s, d, b) \geq 0$. Then as long as 
		$n \geq 4\cdot f_1^2\cdot\max\lp f_2^2, 16\log^2 \lp f_1\rp \rp$,
	$$ \exp\lp -\frac{\sqrt{n}}{f_1(s, d, b)} + f_2(s, d, b) \rp \leq \frac{1}{n}.  $$
\end{lemma}
Taking  $f_1 = {8s}/\sqrt{2^b-1}$, $f_2 = C_0s\log\lp \frac{d}{s} \rp$ and applying Lemma~\ref{lemma:n_requirement}, we see that as long as $n \succeq {s^4\log^2\lp\frac{d}{s}\rp}/{2^b}$, the $\ell_2$ error \eqref{eq:hash_error_1} is $O\lp\max\lp\frac{s}{n 2^b}, \frac{1}{n}\rp\rp$. By \eqref{eq:l_1_2}, we obtain the upper bound on $\ell_1$ error. This gives the second bound on the sample-size requirement in Theorem~\ref{thm:main_smp}.

\subsection{Localization Scheme~C: non-adaptive combinatorial group testing}\label{sec:ncgt}

The non-uniform random hashing scheme presented in Section~\ref{sec:localization_hasing} provides a substantial reduction in the minimum number of samples needed to break the dimension dependence in sample complexity. However, this comes at the expense of increased decoding complexity as it relies on exhaustive search. We now present a group testing based scheme that combines the best of both worlds.  

\paragraph{Group testing preliminaries} Group testing is the problem of identifying $s$ defective items in a large set of cardinality $d$ by making tests on groups of items. A \textit{group test} is applied to a subset of items  $\mcal{S} \subseteq [d]$. The test outcome $Z$ is \textit{positive} (i.e. $Z=1$) if at least one item in $\mcal{S}$ is defective. A group testing algorithm describes how to design $\mcal{S}_1,..., \mcal{S}_T$ and select $T$ such that the defective items can be identified from the test outcomes $Z_1,...,Z_T$. In the non-adaptive setting, all $T$ tests must be designed and fixed before they are conducted. Therefore, each single test $\mcal{S}_t$ can be characterized by a row vector $m_t \in \lbp 0, 1 \rbp^{1\times d}$, where $m_t(j) = 1$ if the $j$-th item is included in the $t$-th test. 
Therefore, the collection of $T$ tests can be represented by a $T\times d$ binary measurement matrix $M = \lb m_1^\intercal,..., m_T ^\intercal \rb^\intercal$. 

The goal of non-adaptive combinatorial group testing (NCGT) is to design the measurement matrix $M$ such that 1) the number of tests $T$ is minimized, and 2) the defective items can be identified correctly (i.e. with \emph{zero-error}) and efficiently (i.e. in $O\lp\msf{poly}\lp \log d\rp \rp$ time). 
In particular, if a matrix $M$ satisfies the \emph{$s$-disjunct} property described below, then a cover decoder (summarized in Algorithm~\ref{alg:dec_localization_3}) can identify all $s$ defective items in $O(Td)$ time.
\begin{definition}[$s$-disjunct]
	Let $M$ be a $T\times d$ binary matrix, $M_j$ be the $j$-th column of $M$, and $\msf{supp}(M_j) \eqDef \lbp t\mv t\in [T],\, M_{t, j} = 1  \rbp$.
	Then $M$ is said to be $s$-disjunct if $\msf{supp}(M_j) \not\subset \bigcup_{j'\in \mcal{K}} \msf{supp}(M_{j'})$, for all $\mcal{K} \subset [d]$ such that 1) $\lba \mcal{K} \rba = s$ and 2) $\mcal{K}$ does not contain $j$.
\end{definition}

Let $T_{\msf{disjunct}}\lp s, d\rp$ denote the minimum number of rows of an $d$-column $s$-disjunct matrix. It has been known for about 40 years \citep{d1982bounds} that when $s = O(\sqrt{d})$,
$$\Omega\lp {s^2 \log d}/{\log s} \rp \leq T_{\msf{disjunct}}\lp s, d\rp \leq O(s^2 \log d). $$

\subsubsection{Support localization via group testing: $1$-bit case}
Next, we map the support localization problem into NCGT for the case $b=1$ by viewing $\mcal{J}_\alpha \subseteq \msf{supp}(p)$ as the defective items that need to be identified (recall that our goal here is to only specify $\mcal{J}_\alpha$ and $|\mcal{J}_\alpha|\leq s$). The main difference between the localization task and group testing is as follows: in group testing, a group test provides information about \emph{all} the defective items participating in the test, while in the localization task each client can observe only a single symbol, and hence makes an observation regarding only one "defective item". 

To match the two problems, we use multiple clients to simulate a single group test. We partition clients into  $T$ bins, and the clients in the $\tau$-th bin encode their observations according to the same group test $\mcal{S}_\tau$. If the clients in the $\tau$-th bin observe all symbols in $\mcal{J}_\alpha$, then by taking the Boolean OR of their reported bits, the server can recover the outcome corresponding to the test $\mcal{S}_\tau$.

\paragraph{Encoding}
Let $M = \lb m_1^\intercal,..., m_T ^\intercal \rb^\intercal\in\lbp 0, 1\rbp^{T\times d}$ be any zero-error NCGT measurement matrix satisfying the $s$-disjunct property. Each client encodes  their local observation according to a row (i.e. an individual group test) of $M$. Define $t(i) \eqDef i \msf{\,mod\,} T$. We then uniformly partition $n_1$ clients into $T$ bins by assigning client $i$ into the $t(i)$-th bin. Client $i$ then generates their $1$-bit report by setting $Y_i = M_{t(i), X_i}$. Equivalently, client $i$'s 1-bit encoding channel matrix $W_i$ is $W_i(y=1|x) = m_{t(i)}$.

\begin{algorithm2e}[H]
		\SetAlgoLined
		\KwIn{$X_i \in [d]$, $M \in \lbp 0, 1 \rbp^{T\times d}$ }
		$t(i) \la i \msf{\,mod\,} T$\;
		$Y_i \la M_{t(i), X_i}$\;
		\KwRet{$Y_i$}
		\caption{Localization via NCGT: $1$-bit encoding }\label{alg:enc_localization_3}
\end{algorithm2e}

\paragraph{Decoding}
Let $\mcal{G}_\tau\eqDef \lbp i \in [n_1]\mv t(i) = \tau \rbp$ denote the $\tau$-th bin of clients. For each $\tau \in [T]$, the server  computes $\hat{Z}_\tau \eqDef \bigvee_{i\in\mcal{G}_\tau} Y_i$. If all symbols in $\mcal{J}_\alpha$ appear at least once in $\mcal{G}_\tau$'s observations (i.e. if $\mcal{J}_\alpha \subseteq \lbp X_i \mv i \in \mcal{G}_\tau \rbp$), then $\hat{Z}_\tau$ is the same as the result of the $\tau$-th group test of $M$ measuring on $\mcal{J}_\alpha$, which we denote by $Z_\tau$. 
Therefore, as long as $n_1$ is large enough, $\bigcup_{\tau \in T}\lbp \hat{Z}_\tau = Z_\tau \rbp$ holds with high probability, and the server can then identify $\mcal{J}_\alpha$ by running a standard cover decoder (which is summarized in Algorithm~\ref{alg:dec_localization_3} below).
\begin{algorithm2e}
		\SetAlgoLined
		\KwIn{$Y^{n_1}$, $M$}
		Initialize $\hat{\mcal{J}}_\alpha = \emptyset$\;
		\For{$\tau \in [T]$}{
			$\hat{Z}_\tau \la \bigvee_{ i: i \equiv \tau \lp \msf{\, mod \,} T\rp } Y_i$ \tcp*{simulate the $\tau$-th group test}
		}
		$\bm{Z} \la \lb \hat{Z}_1,...,\hat{Z}_T \rb^\intercal$\;
		
		\For{$j \in [d]$ \tcp*{run the cover decoder}}{
			\If{$\msf{supp}\lp M_j\rp \subseteq \msf{supp}\lp\bm{Z}\rp$}{
				Add $j$ to $\hat{\mcal{J}}_\alpha$
			}
		}
		\KwRet{$\hat{\mcal{J}}_\alpha$}
		\caption{Localization via NCGT: $1$-bit decoding}\label{alg:dec_localization_3}
\end{algorithm2e}

\begin{lemma}\label{lemma:localization_NCGT_1bit}
	Under the above encoding and decoding schemes (see Algorithm~\ref{alg:enc_localization_3}, \ref{alg:dec_localization_3} for the formal descriptions) with measurement matrix $M \in \lbp 0, 1\rbp^{T\times d}$, we have
	\begin{equation*}
	\Pr\lbp \mcal{J}_\alpha \not\subset \hat{\mcal{J}}_\alpha \rbp \leq \exp\lp -{n_1  \alpha}/{T} + \lp \log s + \log T \rp\rp.
	\end{equation*}
	In addition, with probability $1$, $|\hat{\mcal{J}}_\alpha| \leq s$, and the decoding complexity is $O(n + d T)$.
\end{lemma}

\subsubsection{General $b$-bit case}
For the general $b$-bit case (with $b\leq \log s$), one may attempt to repeatedly apply the $1$-bit encoding scheme for $b$ times. That is, each bin of clients $\mcal{G}_\tau$ simulates $b$ group tests at a time. This can reduce the total number of bins required from $T$ to $T/b$ (and thus increases $\lba \mcal{G}_\tau\rba$ by a factor of $b$), equivalently yielding a boost on the sample size from $n_1$ to $ n_1 b$. 
However, according to Lemma~\ref{lemma:2}, we see that by carefully designing the encoding channels $W_i(y|x)$, it is possible to achieve an \emph{exponential} gain on the sample size, i.e. from $n_1$ to $n_12^b $.  Therefore, our goal is to design the measurement matrix $M$ in a way that each bin of clients $\mcal{G}_\tau$ can simulate $\Theta(2^b)$ group tests at a time. Towards this goal, we want $M$ to have the following properties:
\begin{enumerate}
	\item $M$ is $s$-disjunct (so that the cover decoder applies).
	\item $T = O\lp s^2\cdot \msf{polylog}(d) \rp$.
	\item $M$ is \emph{$(2^b-1)$-blockwise sparse}, that is, for each column $M_j$, every block 
	$$M_j\lp(\tau-1)(2^b-1)+1:\tau(2^b-1)\rp,\, \forall \tau \in \lb \frac{T}{2^b-1}\rb$$
	contains at most one $1$. 
\end{enumerate}
With the $(2^b-1)$ block-wise sparse structure, the $\tau$-th bin of clients select their channel matrices as 
\begin{equation}\label{eq:channel_GT_matrix}
	W_i\lp 1:2^b-1 \mv j\rp = M_j\lp (\tau-1)\cdot (2^b-1)+1:\tau\cdot(2^b-1)\rp, \, \forall j \in [d], \tau \in \lb \frac{T}{2^b-1} \rb.
\end{equation}
Notice that 1) $W_i(1:2^b-1|j)$ determines $W_i(2^b|j)$, and 2) the $(2^b-1)$ block-wise sparse structure ensures that $W_i$'s are valid channels.

To find a measurement matrix $M$ that satisfies the above three criterion, we use the celebrated Kautz and Singleton's construction \citep{kautz1964nonrandom}. This construction uses a $[m, k]_q$ Reed-Solomon code as an outer code $C_{\msf{out}}$ and the identity code $C_\msf{in}$ (i.e. one-hot encoding $I_q:[q] \ra \lbp 0,1\rbp^{q\times 1}$) as the inner code, and the measurement matrix is the concatenation of $C_{\msf{out}}$ and $C_{\msf{in}}$: $M_{\msf{KS}} \eqDef C_{\msf{out}}\circ C_\msf{in} \in \lbp 0, 1\rbp^{mq \times q^k}$. 

For NCGT, we pick $m=q$ (so $T = q^2$ and $d = q^k$) and set the rate $\frac{k}{m} = \frac{1}{s+1}$ to ensure that $M_{\msf{KS}}$ is $s$-disjunct. Thus by selecting $q = \Theta\lp s\log d \rp$, $M_{\msf{KS}}$ satisfies Property~1 and is $\Theta\lp s\log d  \rp$-blockwise sparse (so Property~3 holds for all $b \leq \log\lp s \log d \rp$), with $T_{\msf{KS}} = O\lp s^2\log^2 d \rp$ rows. For more details on Kautz and Singleton's construction, we refer the reader to \citep{inan2019optimality, indyk2010efficiently}. By adopting $M_{\msf{KS}}$ as the encoding channel matrix (as described in \eqref{eq:channel_GT_matrix}), we extend the previous scheme to the $b$-bit setting (see Algorithm~\ref{alg:enc_localization_3b} and Algorithm~\ref{alg:dec_localization_3b} below).

	\begin{algorithm2e}[H]
		\SetAlgoLined
		\KwIn{$X_i \in [d]$, $M_{\msf{KS}} \in \lbp 0, 1 \rbp^{Cs^2\log^2d\times d}$ }
		$t \la i \msf{\,mod\,} \lceil \frac{T}{2^b-1}\rceil$\;
		$Y_i \la M_{(t-1)(2^b-1)+1:t(2^b-1), X_i}$\tcp*{Can be represented in $b$ bits}
		\KwRet{$Y_i$}
		\caption{Localization via NCGT: $b$-bit encoding }\label{alg:enc_localization_3b}
	\end{algorithm2e}
	
	\begin{algorithm2e}
		\SetAlgoLined
		\KwIn{$Y^{n_1}$, $M_\msf{KS}$}
		Initialize $\hat{\mcal{J}}_\alpha = \emptyset$\;
		\For{$\tau \in \lceil \frac{T}{2^b-1} \rceil$}{
			\For{$\kappa \in [2^b-1]$}{
				$\hat{Z}_\tau(\kappa) \la \bigvee_{ i: i \equiv \tau \lp \msf{\, mod \,} \lceil\frac{T}{2^b-1}\rceil\rp } Y_i(\kappa)$ \tcp*{simulate the $\tau$-th group test}
			}
		}
		$\bm{Z} \la \lb \hat{Z}_1,...,\hat{Z}_T \rb^\intercal$\;
		
		\For{$j \in [d]$ \tcp*{run the cover decoder}}{
			\If{$\msf{supp}\lp M_{\msf{KS}}(j)\rp \subseteq \msf{supp}\lp\bm{Z}\rp$}{
				Add $j$ to $\hat{\mcal{J}}_\alpha$
			}
		}
		\KwRet{$\hat{\mcal{J}}_\alpha$}
		\caption{Localization via NCGT: $b$-bit decoding}\label{alg:dec_localization_3b}
	\end{algorithm2e}

\begin{lemma}\label{lemma:localization_NCGT}
	Under the above encoding and decoding schemes (i.e. Algorithm~\ref{alg:enc_localization_3b}, \ref{alg:dec_localization_3b}) with measurement matrix $M_{\msf{KS}}$, we have
	\begin{equation*}
	\Pr\lbp \mcal{J}_\alpha \not\subset \hat{\mcal{J}}_\alpha \rbp \leq \exp\lp -C_0\cdot \frac{n_1 2^b  \alpha}{s^2\log^2 d} + \lp \log s + \log \log d \rp\rp.
	\end{equation*}
	In addition, $|\hat{\mcal{J}}_\alpha| \leq s$ with probability $1$, and the decoding complexity is $O(n2^b + s^2d\log^2 d)$.
\end{lemma}

Picking $n_1 = \frac{n}{2}$, $\alpha = \frac{1}{\sqrt{n2^b}}$, and combining Lemma~\ref{cor:two_stage_error} and Lemma~\ref{lemma:localization_NCGT}, we arrive at \eqref{eq:condition_1} as long as $n \succeq \frac{s^4\log^4\lp{d}/{s}\rp}{2^b}\lp\log s + \log\log d \rp^2$. This proves the second part of Theorem~\ref{thm:main_smp}.

\begin{remark}\label{rmk:efficient}
{
In Algorithm~\ref{alg:dec_localization_3b}, we present a naive NCGT cover decoder, which takes $O\lp d\cdot T_{\msf{KS}} \rp = O(s^2d\log^2 d)$ time. However, since $M_{\msf{KS}}$ is constructed based on Reed-Solomon codes, one can leverage the efficient list recovery algorithm (i.e. Guruswami-Sudan algorithm \citep{guruswami1998improved}) to decode $\hat{J}_\alpha$ in $\msf{poly}(s, \log d)$ time, improving the dependency on $d$ from $\msf{poly}(d)$ to $\msf{poly}(\log d)$.}
\end{remark}

\paragraph{Removing the use of shared randomness} In Scheme~C, the local encoding functions are constructed deterministically according to $M_{\msf{KS}}$ and hence do not involve any randomization, so the only use of shared randomness is in the estimation phase. However, we can circumvent it by considering the following two-round encoding scheme: the second $n_2$ clients encode their observations according to $\hat{\mcal{J}}_\alpha\lp Y^{n_1} \rp$, which can be done, for instance, by using the grouping idea in Algorithm~\ref{alg:enc_localization_1}, but now we only group symbols in $\hat{\mcal{J}}_\alpha$. Unlike the interactive scheme in the next section (i.e. Scheme~D) that requires $\log d$ rounds of interaction, the resulting scheme involves only two-round interaction and no longer needs shared randomness.
\section{Interactive Localization Scheme: A Tree-based Approach}\label{sec:interactive}

Unlike non-interactive localization schemes introduced in the previous section, if we allow for sequential interaction between the server and clients, we can localize $\mcal{J}_\alpha$ more efficiently (i.e. using less samples) and obtain a smaller requirement on the sample size (as described in Theorem~\ref{thm:main_interactive}).  We introduce a tree-based $\log d$-round interactive localization scheme below. 
\paragraph{Sketch of the scheme}
We represent each symbol $j\in[d]$ by a $\log d$ bits binary string, and our algorithm discovers elements in $\mcal{J}_\alpha$ by learning the prefixes of their bit representations sequentially across $\log d$ rounds.
In particular, at each round $t$, the goal is to estimate the set of all length-$t$ prefixes in $\mcal{J}_\alpha$, which we denote by $\mcal{J}_{\alpha, t} \eqDef \lbp \msf{prefix}_t(j) \, \mv \, j \in \mcal{J}_\alpha \rbp$ (so $\mcal{J}_\alpha = \mcal{J}_{\alpha, \log d}$). 

Towards this goal, we first partition $n_1$ clients into $\log d$ equal-sized groups $\mcal{G}_1,...,\mcal{G}_{\log d}$ with clients in $\mcal{G}_t$ participating in round $t$. At round $t$, clients encode their observations according to $\hat{\mcal{J}}_{\alpha, t-1}$, where $\hat{\mcal{J}}_{\alpha, t-1}$ is an estimate of $\mcal{J}_{\alpha, t-1}$ obtained from the previous round. The encoding rule is based on the grouping idea described in Scheme~A, but since now each client has partial knowledge $\hat{\mcal{J}}_{\alpha, t-1}$, they only group symbols whose prefixes lie in $\hat{\mcal{J}}_{\alpha, t-1}$ (instead of grouping the entire $[d]$). This leads to a more efficient way to use the samples and improves the sample size requirement from $O(d^2\log^2 d/2^b)$ to $\tilde{O}(s^2\log^2 d/2^b)$. 

Upon observing $\mcal{G}_t$'s reports, the system (i.e. all subsequent clients and the server) updates $\hat{\mcal{J}}_{\alpha, t-1}$ accordingly to generate $\hat{\mcal{J}}_{\alpha, t}$. When $n_1$ is large enough, this protocol successfully localizes $\mcal{J}_\alpha$ with high probability. We formally state the encoding and the decoding schemes below.

\paragraph{Encoding} Let $\hat{\mcal{J}}_{\alpha, 0} = \emptyset$. At round $t>0$, let $\mcal{G}_t \eqDef \lbp i \in [n_1] \, \mv \, i \equiv t (\,\msf{mod}\, \log d) \rbp$ be the participated clients, and let $\msf{C}\lp \hat{\mcal{J}}_{\alpha,t-1}\rp$ be the set of all candidates of length-$t$ prefixes, that is, 
$$ \msf{C}\lp \hat{\mcal{J}}_{\alpha,t-1}\rp \eqDef \lbp \msf{append}(v, 0), \msf{append}(v, 1) \,\mv\, v \in \hat{\mcal{J}}_{\alpha,t-1} \rbp. $$
If round $t-1$ succeeds, then $\lba \msf{C}\lp \hat{\mcal{J}}_{\alpha,t-1}\rp \rba \leq 2s$. We then partition $\msf{C}\lp \hat{\mcal{J}}_{\alpha,t-1}\rp$ into $M \eqDef \frac{\lba \msf{C}\lp \hat{\mcal{J}}_{\alpha,t-1}\rp \rba}{2^b-1}$ blocks $\mcal{B}_1,...,\mcal{B}_M$, such that each contains $2^b -1$ possible length-$t$ prefixes (assuming $b\leq \log s$). We also partition $\mcal{G}_t$ into $M$ groups $\mcal{K}_{t, 1},...,\mcal{K}_{t, M}$ by setting $\mcal{K}_{t, m }\eqDef \lbp i \in \mcal{G}_t\, \mv \, \frac{i-t}{\log d} \equiv m (\,\msf{mod}\, M) \rbp$.  When client $i$ in $\mcal{K}_{t, m}$ observes $X_i$ with $\msf{prefix}_t(X_i) \in \mcal{B}_m$, they reports the index of $\msf{prefix}_t(X_i)$ in $\mcal{B}_m$ (where we index elements in $\mcal{B}_m$ by $1$ to $2^b-1$), and otherwise they reports $0$. Then the message can be encoded in $b$ bits. Formally, we have
\begin{equation*}
	Y_i = \begin{cases}
		\msf{index}\lp\mcal{B}_m, \msf{prefix}_t(X_i)\rp, &\text{ if } \msf{prefix}_t(X_i) \in \mcal{B}_m,\\
		0, &\text{ else}.
	\end{cases}
\end{equation*}

\paragraph{Decoding} Since each client is assigned to each group $\mcal{K}_{t, m}$ deterministically, from the client's index $i$ and $\hat{\mcal{J}}_{\alpha,t-1}$, one can compute the group index $(t(i), m(i))$ explicitly. Therefore, upon observing $Y_i$, we can estimate $\msf{prefix}_t(X_i)$ by 
$$ \hat{\msf{prf}_t}(Y_i) \eqDef 
\begin{cases}
\text{the } Y_i+(m(i)-1)\cdot\lp 2^b-1\rp \text{-th element in } \msf{C}\lp \hat{\mcal{J}}_{\alpha, t-1}\rp, &\text{ if } Y_i \neq 0,\\
\texttt{null}, &\text{else}.
\end{cases} $$
Finally, we estimate $\mcal{J}_{\alpha,t}$ by $\hat{\mcal{J}}_{\alpha,t}\eqDef \lbp \hat{\msf{prf}_t}(Y_i)  \, \mv\, i \in \mcal{G}_t \rbp$. Under the above encoding and decoding schemes, we have the following bound on the probability of error:

\begin{lemma}\label{lemma:localization_seq}
	Under the above encoding and decoding rules, the failure probability is bounded by
	\begin{equation*}
		\Pr\lbp \mcal{J}_\alpha \not\subset \hat{\mcal{J}}_\alpha \rbp \leq \exp\lp -\frac{n_1(2^b-1)\alpha}{2s\log d} +\lp \log\log d+\log\lp \frac{2s}{2^b-1} \rp\rp \rp.
	\end{equation*}
\end{lemma}

Picking $n_1 = \frac{n}{2}$, $\alpha = \frac{1}{\sqrt{n2^b}}$ and combining Lemma~\ref{lemma:localization_seq} and Lemma~\ref{cor:two_stage_error}, we conclude that as long as $n \succeq \frac{s^2\log^2 d}{2^b}\lp\log s + \log\log d \rp^2$, \eqref{eq:l2_rate} holds. This establishes Theorem~\ref{thm:main_interactive}.


\section{Concluding Remarks and Open Problems}\label{sec:conclusion}

In this work, we characterize the convergence rate of estimating $s$-sparse distributions, showing that by carefully designing encoding and decoding schemes, one can achieve a dimension-free estimation rate, which is the same rate as knowing the sparse support of the distribution beforehand. As a natural next step, we study the threshold on number of samples needed to achieve such a dimension-free estimation rate. We  give upper bounds on this threshold by developing three non-interactive schemes and one interactive scheme. Our results establish an interesting connection between distribution estimation and group testing, suggesting that non-adaptive group testing can be useful in designing efficient decoding and encoding schemes with small sample complexity.

There are several open research directions emerging from our work. First, under the non-interactive model, there exists a gap between our upper and lower bounds on the minimum sample size required $n^*(s, b, d)$ to achieve dimension-free convergence (see the discussion after Corollary~\ref{cor:sc} for more details). We conjecture that the lower bound on $n^*(s, b, d)$ is tight, and the non-interactive schemes can be further improved. Closing this gap remains an open problem. Second, we note that in the estimation phase of our two-stage scheme, we rely on random hash functions to encode local data. It remains unclear whether or not there exists private-coin schemes that achieve the same  estimation error.

Finally, we note an interesting contrast with sparse distribution estimation under local differential privacy (LDP) constraints. There are several recent works in the literature that observe a symmetry between  LDP and communication constraints \citep{chen2020breaking, acharya2019inference, acharya2019inference2,han2018geometric, barnes2019lower,barnes2020fisher}. In particular, without the sparse assumption, previous works \citep{barnes2019lower, barnes2020fisher, acharya2019inference, han2018geometric} show that the minimax $\ell_2$ estimation error under $b$-bit and $\varepsilon$-local differential privacy (LDP) constraints are $\Theta\lp \frac{d}{n2^b} \rp$ and $\Theta\lp \frac{d}{ne^\varepsilon} \rp$ (for $\varepsilon = \Omega(1)$). This implies that in the non-sparse case, compression and LDP have the same sample complexity as long as $b \approx \varepsilon$. This symmetry between communication and LDP constraints has also been observed in other statistical models such as mean estimation.

To the best of our knowledge, the result we derive in this paper is the first to break the symmetry between   communication and privacy constraints in distributed estimation. Under the s-sparse assumption, \citep{acharya2020estimating} shows that $\Theta\lp \frac{s\log(d/s)}{ne^\varepsilon}\rp$ error is fundamental under $\varepsilon$-LDP. Given the symmetry between communication and LDP constraints in previous results as mentioned earlier, one might have been tempted to expect the error under a $b$-bit constraint to be of the form $O\lp \frac{s\log(d/s)}{n2^b}\rp$. Our results suggest that one can achieve $\Theta\lp \frac{s}{n2^b}\rp$ error under a $b$-bit constraint, implying that when estimating sparse distributions, the communication constraint and the LDP constraint behave differently and the latter is strictly more stringent than the former. This loss of symmetry makes it difficult, for example, to postulate the fundamental limit (and how to achieve it) under joint communication and LDP constraints, a direction that has been settled in \citep{chen2020breaking} in the non-sparse case. 
Understanding the convergence rate for sparse distribution estimation  under joint communication and LDP constraints remains an open problem. Further, exploring interactions with secure aggregation is of practical interest, especially in the federated learning and analytics settings.

\acks{The authors would like to thank Professor Mary Wootters for pointing out the efficient decoding algorithm for Kautz and Singleton's group testing scheme, which helped us to improve the time complexity of the localization scheme.}
\newpage
\bibliography{ms}

\begin{thebibliography}{29}
\providecommand{\natexlab}[1]{#1}
\providecommand{\url}[1]{\texttt{#1}}
\expandafter\ifx\csname urlstyle\endcsname\relax
  \providecommand{\doi}[1]{doi: #1}\else
  \providecommand{\doi}{doi: \begingroup \urlstyle{rm}\Url}\fi

\bibitem[Acharya et~al.(2019{\natexlab{a}})Acharya, Canonne, and
  Tyagi]{acharya2019inference}
Jayadev Acharya, Cl{\'e}ment~L Canonne, and Himanshu Tyagi.
\newblock Inference under information constraints: Lower bounds from chi-square
  contraction.
\newblock In \emph{Conference on Learning Theory}, pages 3--17. PMLR,
  2019{\natexlab{a}}.

\bibitem[Acharya et~al.(2019{\natexlab{b}})Acharya, Canonne, and
  Tyagi]{acharya2019inference2}
Jayadev Acharya, Cl{\'e}ment~L Canonne, and Himanshu Tyagi.
\newblock Inference under information constraints ii: Communication constraints
  and shared randomness.
\newblock \emph{arXiv preprint arXiv:1905.08302}, 2019{\natexlab{b}}.

\bibitem[Acharya et~al.(2020)Acharya, Canonne, and Tyagi]{acharya2020general}
Jayadev Acharya, Cl{\'e}ment~L Canonne, and Himanshu Tyagi.
\newblock General lower bounds for interactive high-dimensional estimation
  under information constraints.
\newblock \emph{arXiv preprint arXiv:2010.06562}, 2020.

\bibitem[Acharya et~al.(2021)Acharya, Kairouz, Liu, and
  Sun]{acharya2020estimating}
Jayadev Acharya, Peter Kairouz, Yuhan Liu, and Ziteng Sun.
\newblock Estimating sparse discrete distributions under local privacy and
  communication constraints.
\newblock In \emph{Algorithmic Learning Theory}. PMLR, 2021.

\bibitem[Barlow(1972)]{barlow1972statistical}
Richard~E Barlow.
\newblock Statistical inference under order restrictions; the theory and
  application of isotonic regression.
\newblock Technical report, 1972.

\bibitem[Barnes et~al.(2019)Barnes, Han, and Ozgur]{barnes2019lower}
Leighton~Pate Barnes, Yanjun Han, and Ayfer Ozgur.
\newblock Lower bounds for learning distributions under communication
  constraints via fisher information, 2019.

\bibitem[Barnes et~al.(2020)Barnes, Chen, and Ozgur]{barnes2020fisher}
Leighton~Pate Barnes, Wei-Ning Chen, and Ayfer Ozgur.
\newblock Fisher information under local differential privacy.
\newblock \emph{arXiv preprint arXiv:2005.10783}, 2020.

\bibitem[Bassily and Smith(2015)]{Bassily2015}
Raef Bassily and Adam Smith.
\newblock Local, private, efficient protocols for succinct histograms.
\newblock In \emph{Proceedings of the Forty-Seventh Annual ACM Symposium on
  Theory of Computing}, STOC ’15, page 127–135, New York, NY, USA, 2015.
  Association for Computing Machinery.
\newblock ISBN 9781450335362.
\newblock \doi{10.1145/2746539.2746632}.
\newblock URL \url{https://doi.org/10.1145/2746539.2746632}.

\bibitem[Bassily et~al.(2017)Bassily, Nissim, Stemmer, and
  Thakurta]{Bassily2017}
Raef Bassily, Kobbi Nissim, Uri Stemmer, and Abhradeep Thakurta.
\newblock Practical locally private heavy hitters.
\newblock In \emph{Proceedings of the 31st International Conference on Neural
  Information Processing Systems}, NIPS’17, page 2285–2293, Red Hook, NY,
  USA, 2017. Curran Associates Inc.
\newblock ISBN 9781510860964.

\bibitem[Bun et~al.(2019)Bun, Nelson, and Stemmer]{bun2019heavy}
Mark Bun, Jelani Nelson, and Uri Stemmer.
\newblock Heavy hitters and the structure of local privacy.
\newblock \emph{ACM Transactions on Algorithms (TALG)}, 15\penalty0
  (4):\penalty0 1--40, 2019.

\bibitem[Candes et~al.(2006)Candes, Romberg, and Tao]{candes2006stable}
Emmanuel~J Candes, Justin~K Romberg, and Terence Tao.
\newblock Stable signal recovery from incomplete and inaccurate measurements.
\newblock \emph{Communications on Pure and Applied Mathematics: A Journal
  Issued by the Courant Institute of Mathematical Sciences}, 59\penalty0
  (8):\penalty0 1207--1223, 2006.

\bibitem[Chen et~al.(2020)Chen, Kairouz, and Ozgur]{chen2020breaking}
Wei-Ning Chen, Peter Kairouz, and Ayfer Ozgur.
\newblock Breaking the communication-privacy-accuracy trilemma.
\newblock \emph{Advances in Neural Information Processing Systems}, 33, 2020.

\bibitem[Devroye and G\'abor(1985)]{devroye1985nonparametric}
L.~Devroye and L.~G\'abor.
\newblock \emph{Nonparametric Density Estimation: The L1 View}.
\newblock Wiley Interscience Series in Discrete Mathematics. Wiley, 1985.
\newblock ISBN 9780471816461.
\newblock URL \url{https://books.google.com.tw/books?id=ZVALbrjGpCoC}.

\bibitem[Devroye and Lugosi(2012)]{devroye2012combinatorial}
L.~Devroye and G.~Lugosi.
\newblock \emph{Combinatorial Methods in Density Estimation}.
\newblock Springer Series in Statistics. Springer New York, 2012.
\newblock ISBN 9781461301257.
\newblock URL \url{https://books.google.com.tw/books?id=lQEMCAAAQBAJ}.

\bibitem[Donoho(2006)]{donoho2006compressed}
David~L Donoho.
\newblock Compressed sensing.
\newblock \emph{IEEE Transactions on information theory}, 52\penalty0
  (4):\penalty0 1289--1306, 2006.

\bibitem[Dorfman(1943)]{dorfman1943detection}
Robert Dorfman.
\newblock The detection of defective members of large populations.
\newblock \emph{The Annals of Mathematical Statistics}, 14\penalty0
  (4):\penalty0 436--440, 1943.

\bibitem[Du et~al.(2000)Du, Hwang, and Hwang]{du2000combinatorial}
Dingzhu Du, Frank~K Hwang, and Frank Hwang.
\newblock \emph{Combinatorial group testing and its applications}, volume~12.
\newblock World Scientific, 2000.

\bibitem[D'yachkov and Rykov(1982)]{d1982bounds}
Arkadii~Georgievich D'yachkov and Vladimir~Vasil'evich Rykov.
\newblock Bounds on the length of disjunctive codes.
\newblock \emph{Problemy Peredachi Informatsii}, 18\penalty0 (3):\penalty0
  7--13, 1982.

\bibitem[Guruswami and Sudan(1998)]{guruswami1998improved}
Venkatesan Guruswami and Madhu Sudan.
\newblock Improved decoding of reed-solomon and algebraic-geometric codes.
\newblock In \emph{Proceedings 39th Annual Symposium on Foundations of Computer
  Science (Cat. No. 98CB36280)}, pages 28--37. IEEE, 1998.

\bibitem[Han et~al.(2018{\natexlab{a}})Han, Mukherjee, Ozgur, and
  Weissman]{han2018distributed}
Yanjun Han, Pritam Mukherjee, Ayfer Ozgur, and Tsachy Weissman.
\newblock Distributed statistical estimation of high-dimensional and
  nonparametric distributions.
\newblock In \emph{2018 IEEE International Symposium on Information Theory
  (ISIT)}, pages 506--510. IEEE, 2018{\natexlab{a}}.

\bibitem[Han et~al.(2018{\natexlab{b}})Han, {\"O}zg{\"u}r, and
  Weissman]{han2018geometric}
Yanjun Han, Ayfer {\"O}zg{\"u}r, and Tsachy Weissman.
\newblock Geometric lower bounds for distributed parameter estimation under
  communication constraints.
\newblock \emph{arXiv preprint arXiv:1802.08417}, 2018{\natexlab{b}}.

\bibitem[Inan et~al.(2019)Inan, Kairouz, Wootters, and
  {\"O}zg{\"u}r]{inan2019optimality}
Huseyin~A Inan, Peter Kairouz, Mary Wootters, and Ayfer {\"O}zg{\"u}r.
\newblock On the optimality of the kautz-singleton construction in
  probabilistic group testing.
\newblock \emph{IEEE Transactions on Information Theory}, 65\penalty0
  (9):\penalty0 5592--5603, 2019.

\bibitem[Indyk et~al.(2010)Indyk, Ngo, and Rudra]{indyk2010efficiently}
Piotr Indyk, Hung~Q Ngo, and Atri Rudra.
\newblock Efficiently decodable non-adaptive group testing.
\newblock In \emph{Proceedings of the twenty-first annual ACM-SIAM symposium on
  Discrete Algorithms}, pages 1126--1142. SIAM, 2010.

\bibitem[Kairouz et~al.(2019)Kairouz, McMahan, Avent, Bellet, Bennis, Bhagoji,
  Bonawitz, Charles, Cormode, Cummings, et~al.]{kairouz2019advances}
Peter Kairouz, H~Brendan McMahan, Brendan Avent, Aur{\'e}lien Bellet, Mehdi
  Bennis, Arjun~Nitin Bhagoji, Keith Bonawitz, Zachary Charles, Graham Cormode,
  Rachel Cummings, et~al.
\newblock Advances and open problems in federated learning.
\newblock \emph{arXiv preprint arXiv:1912.04977}, 2019.

\bibitem[Kautz and Singleton(1964)]{kautz1964nonrandom}
William Kautz and Roy Singleton.
\newblock Nonrandom binary superimposed codes.
\newblock \emph{IEEE Transactions on Information Theory}, 10\penalty0
  (4):\penalty0 363--377, 1964.

\bibitem[Ngo and Du(2000)]{ngo2000survey}
Hung~Q Ngo and Ding-Zhu Du.
\newblock A survey on combinatorial group testing algorithms with applications
  to dna library screening.
\newblock \emph{Discrete mathematical problems with medical applications},
  55:\penalty0 171--182, 2000.

\bibitem[Pearson(1894)]{pearson1894contributions}
Karl Pearson.
\newblock Contributions to the mathematical theory of evolution.
\newblock \emph{Philosophical Transactions of the Royal Society of London. A},
  185:\penalty0 71--110, 1894.

\bibitem[Silverman(1986)]{Silverman86}
B.~W. Silverman.
\newblock \emph{Density Estimation for Statistics and Data Analysis}.
\newblock Chapman \& Hall, London, 1986.

\bibitem[Zhu et~al.(2020)Zhu, Kairouz, McMahan, Sun, and Li]{zhu2020federated}
Wennan Zhu, Peter Kairouz, Brendan McMahan, Haicheng Sun, and Wei Li.
\newblock Federated heavy hitters discovery with differential privacy.
\newblock In \emph{International Conference on Artificial Intelligence and
  Statistics}, pages 3837--3847. PMLR, 2020.

\end{thebibliography}

\newpage

\appendix

\section{Estimation stage: random hashing}\label{sec:estimation_stage}
\paragraph{Encoding} For the second group of clients $i \in [n_1+1:n]$, they report their local samples through $n_2$ independent random hash functions $h_{n_1+1},...,h_{n}$(which are generated via shared randomness). Equivalently, client $i$'s encoding channel $W_i(y|x)$ is constructed as follows: each column of the channel matrix $W_i$ can be viewed as the one-hot representation of $L_{i,x}$, where 
$$ L_{i,x}\diid \msf{uniform}\lp 2^b \rp, \,\, \forall i\in[n_1+1:n], \, x \in [d].$$
Formally,
$$ W_i(\cdot | x) \eqDef \lb \bbm{1}_{\lbp L_{i,x} = 1 \rbp}, ..., \bbm{1}_{\lbp L_{i,x} = 2^b \rbp} \rb^\intercal. $$
Note that since $W_i \in \{0,1\}^{2^b \times d}$, the local encoder is \emph{deterministic}, so we can also write  $ Y_i = h_i\lp X_i \rp$ for some deterministic hash function $h_i(x)$. 

\paragraph{Decoding}
Upon obtaining $\hat{\mcal{J}}_\alpha$ from the first stage and receiving $Y^{n_2}$\footnote{With a slight abuse of notation, we use $Y^{n_2}$ to denote the collection of $\lp Y_{n_1+1}, ..., Y_n \rp$.}, the server computes count on each symbol $j$ $N_j(Y^{n_2})\eqDef \lba \lbp i\in [n_1+1:n]: h_i(j) = Y_i \rbp \rba$. Note that $\Pr\lbp h_i(j) = Y_i \rbp = \frac{p_j\lp2^b -1\rp + 1}{2^b} \eqDef b_j$. The final estimator $\hat{p}_j\lp Y^{n_2} \rp$ is then defined as
\begin{equation}\label{eq:p_estimator}
\hat{p}_j\lp Y^{n_2} \rp = 
\begin{cases}
\frac{N_j\lp 2^b-1\rp}{n_2\cdot 2^b}-\frac{1}{2^b}, \, &\text{if } j \in \hat{\mcal{J}}_\alpha\\
0,\, & \text{else}.
\end{cases}
\end{equation}

We summarize the encoding algorithm in Algorithm~\ref{alg:enc_estimation}. The decoding algorithm (Algorithm~\ref{alg:dec_estimation}) is given below.

\begin{algorithm2e}[H]
		\SetAlgoLined
		\KwIn{$Y^{n_2}$, $b \in \mbb{N}$, $\hat{\mcal{J}}_\alpha$ }
		\KwOut{$\hat{p}$}
		Set $\hat{p} = \bm{0}$\;
		\For{ $j \in \hat{\mcal{J}}_\alpha$}{
			$N_j(Y^{n_2})\eqDef \lba \lbp i\in [n_1+1:n]: h_i(j) = Y_i \rbp \rba$\;
		}
		\For{ $j \in \hat{\mcal{J}}_\alpha$}{
			$\hat{p}_j\lp Y^{n} \rp = 
			\frac{N_j\lp 2^b-1\rp}{n_2\cdot 2^b}-\frac{1}{2^b}$\;
		}
		\KwRet{$\hat{p}$}
		\caption{Decoding for the estimation phase}\label{alg:dec_estimation}
\end{algorithm2e}

\paragraph{Bounds on the estimation error} Let $\mcal{E}_g$ be the event that the localization phase succeeds, i.e.
$\mcal{E}_g \eqDef  \lbp \mcal{J}_\alpha \subseteq \hat{\mcal{J}}_\alpha\rbp$. Then conditioned on $\mcal{E}_g$, we can control the $\ell_2$ estimation error as follows: 
\begin{lemma}\label{lemma:1}
	Let $\mcal{E}_g$ and $\hat{p}$ be defined as above. Then conditioned on $\mcal{E}_g$, we have
	$$ \E\lb \sum_{j\in [d]}\lp\hat{p}_j\lp Y^{n_2} \rp - p_j\rp^2 \mv \mcal{E}_g\rb \leq s\alpha^2+\frac{s}{n_22^b}+\frac{1}{n_2}. $$
\end{lemma}
\begin{proof}
	The proofs of Lemma~\ref{lemma:1} can be found in Section~\ref{sec:proofs}.
\end{proof}

\section{Missing Proofs}\label{sec:proofs}

\subsection*{Proof of Lemma~\ref{lemma:1}}
\begin{proof}
	Notice that $N_j \sim \msf{Binom}\lp n_2, b_j \rp$, so for all $j \in \hat{\mcal{J}}_\alpha$, $\hat{p}_j\lp Y^n \rp$ yields an unbiased estimator on $p_j$. Moreover, 
	\begin{equation*}
	\E \lb \lp \hat{p}_j - p_j\rp^2 \mv j \in \hat{\mcal{J}}_\alpha\rb = \Var \lp \hat{p}_j \mv j \in \hat{\mcal{J}}_\alpha \rp \leq \frac{1}{n_2^2}\Var\lp \msf{Binom}\lp n_2, b_j \rp \rp = \frac{b_j}{n_2}.
	\end{equation*}
	Summing over $j$, we obtain
	$$	\sum_{j\in \hat{\mcal{J}}_\alpha} \E \lb \lp \hat{p}_j - p_j\rp^2 \mv j \in \hat{\mcal{J}}_\alpha \rb \leq \frac{1+\frac{s}{2^b}}{n_2}, $$
	with probability $1$ since $\lba \hat{\mcal{J}}_\alpha \rba \leq s$ almost surely. This implies
	$$ \E\lb\sum_{j \in \mcal{J_\alpha}}\lp \hat{p}_j - p_j \rp^2\mv \mcal{E}_g\rb \leq \frac{s+2^b}{n_2 2^b}. $$
	Together with the fact that
	$$ \E\lb\sum_{j \notin \mcal{J_\alpha}}p_j^2\mv \mcal{E}_g\rb \leq s\alpha^2, $$
	we establish Lemma~\ref{lemma:1}.
\end{proof}
\subsection*{Proof of Lemma~\ref{cor:two_stage_error}}
\begin{proof}
	Observe that
	\begin{align*}
	\E\lb \sum_{j\in[d]} \lp \hat{p}_j -p\rp^2 \rb
	& = \Pr\lbp \mcal{E}_g^c \rbp\cdot \E\lb \sum_{j\in [d]}\lp\hat{p}_j\lp Y^{n} \rp - p_j\rp^2 \mv \mcal{E}^c_g\rb + \Pr\lbp \mcal{E}_g \rbp\cdot\E\lb \sum_{j\in [d]}\lp\hat{p}_j\lp Y^{n} \rp - p_j\rp^2 \mv \mcal{E}_g\rb \nonumber\\
	& \overset{\text{(a)}}{\leq} 2\Pr\lbp \mcal{E}_g^c \rbp + \E\lb \sum_{j\in [d]}\lp\hat{p}_j\lp Y^{n} \rp - p_j\rp^2 \mv \mcal{E}_g\rb \nonumber\\
	& \overset{\text{(b)}}{\leq} 2\Pr\lbp \mcal{E}_g^c \rbp+s\alpha^2+\frac{s}{n_22^b}+\frac{1}{n_2},
	\end{align*}
	where (a) is due to the fact that 
	$$ \E\lb\sum_{j\in [d]}\lp\hat{p}_j\lp Y^{n} \rp - p_j\rp^2\mv \mcal{E}^c_g\rb \leq \E\lb\sum_{j\in [d]} \hat{p}^2_j\lp Y^{n} \rp + p^2_j\mv \mcal{E}^c_g\rb \leq E\lb\sum_{j\in [d]}\hat{p}_j\lp Y^{n} \rp + p_j\mv \mcal{E}^c_g\rb \leq 2,$$
	and (b) is due to Lemma~\ref{lemma:1}.
	\end{proof}

\subsection*{Proof of Lemma~\ref{clm:1}}
\begin{proof}
	If $j \in \lb (m-1)\cdot(2^b-1)+1: m\cdot(2^b-1)-1\rb \cap \mcal{J}_\alpha$, symbol $j$ can be successfully included in $\hat{\mcal{J}}_\alpha$ only when at least one of clients in $\mcal{G}_m$ observes it. Therefore, let
	$\mcal{E}_j \eqDef \lbp X_j \neq j, \, \forall i\in\mcal{G}_m\rbp$ be the event of failing to include symbol $j$, then
	$$ \Pr\lbp \hat{\mcal{J}}_\alpha\subseteq \mcal{J}_\alpha   \rbp = 1-\Pr\lbp \bigcup_{j\in \mcal{J}_\alpha}\mcal{E}_j \rbp \geq 1-\lba \mcal{J}_\alpha \rba\cdot\lp1-\alpha\rp^{n'} \geq 1-s\cdot e^{-n'\alpha}, $$
	where $n' \eqDef \lba \mcal{G}_m \rba = \frac{n_1(2^b-1)}{d}$.
\end{proof}

\subsection*{Proof of Lemma~\ref{lemma:2}}
\begin{proof}
	Let $\mcal{E}_\ell$ be the event such that 1) $\mcal{J}_\alpha \not\subset \mcal{C}_\ell$ for some $\ell \in [N]$ and 2) $\mcal{C}_\ell$ is consistent with $\lp W^n, Y^n \rp$. Then the error $\mcal{J}_\alpha  \not\subset \hat{\mcal{J}}_\ell$ can happen only if $\bigcup_{\ell \in [N]} \mcal{E}_\ell$ occurs. Hence it suffices to control the probability of $\mcal{E}_\ell$ and then apply union bound.
	
	To bound $\mcal{E}_\ell$, denote $j \in \mcal{J}_\alpha \setminus \mcal{C}_\ell$. Note that as long as the server ensures a client observes symbol $j$, they can rule out $\mcal{C}_\ell$ from the candidates set. However, since each client only reports the hash value of their observation, the server cannot directly obtain such information. To address the difficulty, the following definition describes the condition that makes a channel $W$ ``good'' with respect to $\mcal{C}_\ell$ and $j$:
	\begin{definition}\label{def:distinguish}
		We call a channel $W$ \emph{distinguishes} $j$ under $H_\ell$, if for all $j' \in \mcal{C}_\ell$, we have $W(\cdot|j') \neq W(\cdot|j)$\footnote{Note that due to our construction, $W$ is \emph{deterministic}, i.e. $W(\cdot|j) = e_l$ for some $l \in [2^b]$.}.
		\end{definition}
		
		Then $\mcal{E}_\ell$ cannot happen if there exist a client $i$ 1) who observes $j$ and 2) whose channel $W_i$ distinguishes $j$. Notice that due to our construction of localization channels,
		\begin{align}\label{eq:distinguish_j}
		\Pr\lbp W_i \text{ distinguishes } j \mv H_\ell \rbp 
		& = \Pr\lbp \forall j' \in \mcal{C}_\ell,\, W_i\lp \cdot|j \rp \neq W_i\lp \cdot|j' \rp \rbp \nonumber\\
		& \overset{\text{(a)}}{=} \sum_{y \in [2^b]} \Pr\lbp L_{i, j} = y  \rbp \cdot \prod_{j' \in \mcal{C}_\ell}\Pr\lbp L_{i,j} \neq y \rbp \nonumber\\
		& \geq \frac{2^b-1}{s}\cdot\lp 1- \frac{1}{s}\rp^s \nonumber\\
		& \overset{\text{(b)}}{\geq} \frac{2^b-1}{4s},
		\end{align}
		where $L_{i,j}$ in (a) is defined in \eqref{eq:Lx}, and (b) is due to the fact that $f(s) \eqDef \lp 1- \frac{1}{s}\rp^s$ increasing in $s$ and $s \geq 2$.
		We also have $ \Pr\lbp X_i = j \rbp \geq \alpha,$ and since we generate $W_i$ independently,
		$$ \Pr\lbp \lbp X_i = j \rbp \cap \lbp W_i \text{ distinguishes } j \rbp \rbp \geq \alpha \cdot \frac{2^b-1}{4s}. $$
		Thus we can upper bound the probability of error by 
		$$ \Pr\lbp \mcal{E}_\ell \rbp \leq \lp 1- \alpha \cdot \frac{2^b-1}{4s}\rp^{n_1} \leq \exp\lp -{n_1}  \alpha\frac{2^b-1}{4s}\rp. $$
		Finally applying union bound, we arrive at
		$$ \Pr\lbp \bigcup_{\ell\in[N]}\mcal{E}_\ell \rbp \leq {d\choose s}\cdot\exp\lp -n_1  \alpha\frac{2^b-1}{4s}\rp \leq \exp\lp -n_1  \alpha\frac{2^b-1}{4s} + C_0s\log\lp \frac{d}{s}\rp\rp, $$
		establishing the lemma.
\end{proof}

\subsection*{Proof of Lemma~\ref{lemma:n_requirement}}
\begin{proof}
	
	Observe that the condition 
	\begin{equation}\label{eq:condition_1}
	n \geq 4\cdot f_1^2\cdot\max\lp f_2^2, 16\log^2 \lp f_1\rp \rp
	\end{equation}
	implies 
	$$ \frac{n}{4} \geq f_1^2 \cdot f_2^2 \Longrightarrow \frac{\sqrt{n}}{2} \geq f_1\cdot f_2.$$
	\eqref{eq:condition_1} also implies
	$$ n \geq 64\cdot f_1^2 \cdot \log\lp f_1 \rp^2 \overset{\text{(b)}}{\Longrightarrow} \frac{\sqrt{n}}{\log n} \geq \frac{8\cdot f_1\cdot \log f_1}{2\log f_1 +2\log \log f_1 +2\log 64} \geq 2f_1 \Longrightarrow  \frac{\sqrt{n}}{2} \geq f_1 \cdot \log n,$$
	where (b) holds since 1) $\frac{\sqrt{n}}{\log n}$ is increasing when $n \geq 10$ (note that by assumption, $f_1\geq 300$, so $4\cdot f^2_1\cdot \log f^2_1 > 10$), and 2) $\log f_1 +\log \log f_1 +\log 64 \leq 2 \log f_1 $ when $f_1 \geq 300$.
	Therefore we have
	$$\sqrt{n} \geq  f_1 \cdot \log n + f_1\cdot f_2 \Longrightarrow \exp\lp -\frac{\sqrt{n}}{f_1} + f_2 \rp \leq -\frac{1}{n}. $$
\end{proof}

\subsection*{Proof of Lemma~\ref{lemma:localization_NCGT_1bit}}
\begin{proof}
	Recall that $\mcal{G_\tau}$ is the $\tau$-th bin of clients. Notice that as long as the decoding succeeds ( i.e. $\lbp \mcal{J}_\alpha \not\subset \hat{\mcal{J}}_\alpha \rbp$) if for all $\tau \in [T]$, all symbols in $\mcal{J}_\alpha$ appear at least once in $\mcal{G}_\tau$'s observation. Let $\mcal{E}_\tau$ denotes the error event $ \mcal{E}_\tau \eqDef \lbp \mcal{J}_\alpha \not\subset \lbp X_i \mv i\in\mcal{G}_\tau  \rbp\rbp.$
	Then we have 
	$$\Pr\lbp \mcal{J}_\alpha \not\subset \hat{\mcal{J}}_\alpha \rbp \leq \Pr\lbp \bigcup_{\tau \in [T]}\mcal{E}_\tau\rbp \leq \sum_{\tau \in [T]} \Pr\lbp \mcal{E}_\tau \rbp,$$ so it suffices to lower bound the probability of $\mcal{E}_\tau$.
	
	Recall that for each symbol $j \in \mcal{J}_\alpha$, $\Pr\lbp X_i = j \rbp \geq \alpha$, and we also have $\lba \mcal{J}_\alpha \rba \leq s$. Therefore by union bound,
	\begin{align*}
	\Pr\lbp \mcal{E}_\tau \rbp \leq \lba \mcal{J}_\alpha \rba\cdot(1-\alpha)^{\frac{n_1}{T}} \leq \exp\lp  -\frac{n_1  \alpha}{T} + \log s  \rp.
	\end{align*}
	Finally, applying union bound on $\tau \in [T]$ again, we arrive at the desired bound.
\end{proof}

\subsection*{Proof of Lemma~\ref{lemma:localization_NCGT}}
\begin{proof}
	The proof is the same as in Lemma~\ref{lemma:localization_NCGT_1bit}, except for replacing $\lba \mcal{G}_\tau \rba$ from $n_1/T$ to $n_12^b/T_{\msf{KS}}$.
\end{proof}

\subsection*{Proof of Lemma~\ref{lemma:localization_seq}}
\begin{proof}
	Let $T \eqDef \log d$ and $\mcal{E}_t$ be the event that round $t$ succeeds for all $1 \leq t \leq T$. Then $\lbp \mcal{J}_\alpha \not\subset \hat{\mcal{J}}_\alpha \rbp = \mcal{E}^c_{T}$. By union bound, we have
	\begin{equation}\label{eq:e_seq}
	\Pr\lbp  \mcal{E}^c_{T} \rbp \leq \Pr\lbp\mcal{E}_1\rbp+\sum_{t=2}^T\Pr\lbp \mcal{E}^c_{t} \cap \mcal{E}_{t-1} \rbp \leq \sum_{t=1}^T\Pr\lbp \mcal{E}^c_{t} \mv \mcal{E}_{t-1} \rbp,
	\end{equation}
	where $\mcal{E}_0$ always holds since the length-$0$ prefix set is always empty. Therefore, it suffices to control $\Pr\lbp \mcal{E}^c_{t} \mv \mcal{E}_{t-1} \rbp$.
	
	Note that according to our decoding rule, every element in $\hat{\mcal{J}}_{\alpha,t}$ must be a length-$t$ prefix of symbols in $\msf{supp}(p)$, so we always have $\lba \hat{\mcal{J}}_{\alpha,t} \rba \leq \lba\msf{supp}(p)\rba \leq  s$. Hence to bound the failure probability $\Pr\lbp \mcal{E}^c_{t} \mv \mcal{E}_{t-1} \rbp$, we only need to control $\Pr \lbp \mcal{J}_{\alpha, t} \not\subset \hat{\mcal{J}}_{\alpha, t}  \mv \mcal{J}_{\alpha, t-1} \subseteq \hat{\mcal{J}}_{\alpha, t-1} \rbp$.
	
	If $\mcal{J}_{\alpha, t-1} \subseteq \hat{\mcal{J}}_{\alpha, t-1}$ holds, then $\mcal{J}_{\alpha, t}$ must be contained in the candidate set $\msf{C}\lp\hat{\mcal{J}}_{\alpha, t}\rp$, so the only way that $\lbp \mcal{J}_{\alpha, t} \not\subset \hat{\mcal{J}}_{\alpha, t}\rbp$ can happen is that there is some symbol in $\mcal{B}_m \cap \mcal{J}_{\alpha, t}$ that is not observed by $\mcal{K}_{t,m}$. Therefore, we have 
	\begin{align}\label{eq:interactive_err_bdd}
		\Pr\lbp \mcal{E}^c_{t} \mv \mcal{E}_{t-1} \rbp
		&= \Pr \lbp \mcal{J}_{\alpha, t} \not\subset \hat{\mcal{J}}_{\alpha, t}  \mv \mcal{J}_{\alpha, t-1} \subseteq \hat{\mcal{J}}_{\alpha, t-1} \rbp 
		\leq \Pr\lbp \exists j \in \mcal{B}_m \text{ s.t. } \forall i \in \mcal{K}_{m, t}\, X_i \neq j \rbp \nonumber\\
		& \leq \lba \mcal{B}_m \rba\cdot \lp 1-\alpha\rp^{\lba\mcal{K}_{m,t}\rba}
		\leq \frac{2s}{2^b-1}\cdot\exp\lp -\frac{n_1\alpha}{2s\log d} \rp,
	\end{align}
	where the last inequality follows from $\lba\mcal{K}_{m,t}\rba = \frac{\lba \mcal{G}_t \rba}{M} \geq n_1\frac{2^b-1}{2s\log d}$. 
	Plugging into \eqref{eq:e_seq}, we obtain
	$$ \Pr\lbp  \mcal{E}^c_{T} \rbp \leq  T\cdot \frac{2s}{2^b-1}\cdot\exp\lp -\frac{n_1\alpha}{2s\log d} \rp = \exp\lp -\frac{n_1(2^b-1)\alpha}{2s\log d} +\lp \log\log d+\log\lp \frac{2s}{2^b-1} \rp\rp \rp.$$
\end{proof}

\subsection*{Proof of Theorem~\ref{thm:freq_est}}
\begin{proof}
We prove that if we first randomly permute all of $n$ clients with shared randomness, then all of previous schemes apply to distribution-free setting. We start with introducing a few notation. Let $f_1, f_2,...,f_d \in [n]$ be the empirical frequencies of each symbol, i.e. $f_j \eqDef n\pi_j$, and let $\sigma$ be a $n$-permutation drawn uniformly at random from the permutation group $\mcal{S}_n$ with shared randomness. We set $n_1 = n_2 = n/2$ in the two-stage generic scheme Algorithm~\ref{alg:two_stage}, and use $F_1,...,F_d$ to denote the empirical frequency of the second half of samples, i.e. 
$F_j \eqDef \sum_{i=n/2+1}^n \bbm{1}_{\lbp X_{\sigma(i)} = j\rbp}$. Notice that $F_j$ is a random variable (since $\sigma$ is random) with hyper-geometric distribution $\msf{HG}\lp n, \frac{n}{2}, f_j \rp$. Finally, let $\Pi_j \eqDef \frac{2F_j}{n}$ be the empirical distribution of the second half of clients.
	
Let $\hat{\mcal{J}}_\alpha \lp Y^{n_1} \rp$ be an estimate of $\mcal{J}_\alpha \eqDef \lbp j \in [d] \,\mv\, \pi_j \geq \alpha \rbp$ and $\hat{\Pi}_j\lp Y^n_2 \rp$ be an estimator of $\Pi_j$ (both will be explicitly defined later), then the final estimator $\hat{p}_j$ is defined in the same way as \eqref{eq:p_estimator}, i.e.
$$ \hat{\pi}_j \eqDef \hat{\Pi}_j \cdot \bbm{1}_{\lbp j\in\hat{\mcal{J}}_\alpha \rbp}. $$
Now, consider using the same estimation scheme defined in Section~\ref{sec:estimation_stage} (i.e. Algorithm~\ref{alg:enc_estimation} and Algorithm~\ref{alg:dec_estimation}) with clients' index being replaced by $\sigma(i)$. Then $\hat{\Pi}_j\lp Y^{n_2} \rp$ is
\begin{equation}
\hat{\Pi}_j\lp Y^{n_2} \rp = \frac{2\cdot 2^b}{n\lp2^b-1\rp}\sum_{i=\frac{n}{2}+1}^n\bbm{1}_{\lbp Y_i = h_i\lp j \rp \rbp} - \frac{1}{2^b}.
\end{equation}
Notice that condition on $\sigma$, $\bbm{1}_{\lbp Y_i = h_i\lp j \rp \rbp}$ follows distribution $\bbm{1}_{\lbp X_i = j \rbp}+ \bbm{1}_{\lbp X_i \neq j \rbp} \cdot \msf{Ber}\lp \frac{1}{2^b} \rp$, so 
$$\sum_{i=\frac{n}{2}+1}^n\bbm{1}_{\lbp Y_i = h_i\lp j \rp \rbp}\sim F_j+\msf{Binom}\lp\frac{n}{2}-F_j, \frac{1}{2^b}\rp.$$ Thus $\hat{\Pi}_j\lp Y^n_2 \rp$ yields an unbiased estimator on $\Pi_j$ with variance bounded by
\begin{align}\label{eq:Pi_var_bdd}
	\Var\lp \hat{\Pi}_j \mv\sigma\rp 
	& = \Var\lp \frac{2\cdot 2^b}{n\lp2^b-1\rp}\msf{Binom}\lp\frac{n}{2}-F_j, \frac{1}{2^b}\rp \mv\sigma\rp \nonumber\\
	& \leq \frac{4\cdot 4^b}{n^2 \lp2^b-1\rp^2}\cdot \lp \frac{n}{2}-F_j \rp\cdot \frac{1}{2^b}
	\leq \frac{4}{n2^b}.
\end{align}
		
Next, we control the estimation errors by 
\begin{align}\label{eq:var_bdd}
		\E\lb \sum_{j\in[d]} \lp \hat{\pi}_j - \pi_j \rp^2 \rb 
		& \leq 2\E\lb \sum_{j\in[d]} \lp \hat{\pi}_j - \Pi_j \rp^2 \rb +2\E\lb \sum_{j\in[d]} \lp \Pi_j - \pi_j \rp^2 \rb.
\end{align}
To bound the first term, consider two cases $j \in \mcal{J}_\alpha$ and $j \not\in\mcal{J}_\alpha$. For $j \in \mcal{J}_\alpha$, we have
\begin{align*}
				&\E\lb \lp \hat{\pi}_j - \Pi_j \rp^2 \mv \sigma\rb \\
				= &\Pr\lbp \mcal{J}_\alpha \subseteq \hat{\mcal{J}}_\alpha \mv \sigma \rbp 
				\cdot \E\lb \lp \hat{\Pi}_j - \Pi_j \rp^2 \mv \mcal{J}_\alpha \subseteq \hat{\mcal{J}}_\alpha , \sigma \rb +
				\Pr\lbp \mcal{J}_\alpha \not\subset \hat{\mcal{J}}_\alpha \mv \sigma \rbp 
				\cdot \E\lb \lp \hat{\pi}_j - \Pi_j \rp^2 \mv \mcal{J}_\alpha \not\subset \hat{\mcal{J}}_\alpha , \sigma \rb \\
				\overset{\text{(a)}}{=}&\Pr\lbp \mcal{J}_\alpha \subseteq \hat{\mcal{J}}_\alpha \mv \sigma \rbp 
				\cdot \E\lb \lp \hat{\Pi}_j - \Pi_j \rp^2 \mv \sigma \rb +
				\Pr\lbp \mcal{J}_\alpha \not\subset \hat{\mcal{J}}_\alpha \mv \sigma \rbp 
				\cdot \E\lb \lp \hat{\pi}_j - \Pi_j \rp^2 \mv \mcal{J}_\alpha \not\subset \hat{\mcal{J}}_\alpha , \sigma \rb \\
				\leq &\E\lb \lp \hat{\Pi}_j - \Pi_j \rp^2 \mv \sigma \rb+2\Pr\lbp \mcal{J}_\alpha \not\subset \hat{\mcal{J}}_\alpha \mv \sigma \rbp 
				\cdot \E\lb \hat{\pi}^2_j + \Pi_j^2 \mv \mcal{J}_\alpha \not\subset \hat{\mcal{J}}_\alpha , \sigma \rb,
\end{align*}
where (a) holds since conditioned on $\sigma$, $\hat{\Pi}$ is independent with $\hat{\mcal{J}}_\alpha$.
Summing  over $j$, we obtain
\begin{align}\label{eq:Pi_bdd_1}
	\E\lb \sum_{j\in\mcal{J}_\alpha} \lp \hat{\pi}_j - \Pi_j \rp^2 \mv \sigma \rb 
	& \leq \sum_{j \in \mcal{J}_\alpha} \Var\lp \hat{\Pi}_j \mv\sigma\rp + 2\Pr\lbp \mcal{J}_\alpha \not\subset \hat{\mcal{J}}_\alpha \mv \sigma \rbp \sum_{j \in \mcal{J}_\alpha} \E\lb \hat{\pi}^2_j + \Pi_j^2 \mv \mcal{J}_\alpha \not\subset \hat{\mcal{J}}_\alpha , \sigma \rb\nonumber\\
	&\leq \frac{4s}{n2^b} + 4\Pr\lbp \mcal{J}_\alpha \not\subset \hat{\mcal{J}}_\alpha \mv \sigma \rbp,
\end{align}
where the last inequality follows from \eqref{eq:Pi_var_bdd} and the fact that $\sum_{j} \Pi_j^2 \leq 1$ and $\sum_{j} \hat{\pi}_j^2 \leq 1$. Similarly for $j\not\in\mcal{J}_\alpha$, we have
	\begin{align*}
	&\E\lb \lp \hat{\pi}_j - \Pi_j \rp^2 \mv \sigma\rb 
	\leq \Pr\lbp j \in \hat{\mcal{J}}_\alpha \mv \sigma\rbp \cdot \E\lb \lp \hat{\Pi}_j - \Pi_j \rp^2 \mv \sigma \rb+\Pr\lbp j \not\in \hat{\mcal{J}}_\alpha \mv \sigma \rbp \cdot \Pi_j^2.
	\end{align*}
Summing  over $j$, we obtain
\begin{align}\label{eq:Pi_bdd_2}
	\E\lb \sum_{j\not\in\mcal{J}_\alpha} \lp \hat{\pi}_j - \Pi_j \rp^2 \mv \sigma \rb 
	& \leq \frac{4}{n2^b}\sum_{j\not\in \mcal{J}_\alpha}\Pr\lbp j \in \hat{\mcal{J}}_\alpha \mv \sigma\rbp + \sum_{j \not\in \mcal{J}_\alpha} \Pi_j^2\nonumber \\
	& \leq \frac{4}{n2^b}\E\lb \sum_{j \in [d]} \bbm{1}_{\lbp j \in \hat{\mcal{J}}_\alpha \rbp} \rb + \sum_{j \not\in \mcal{J}_\alpha} \Pi_j^2 \nonumber\\
	& \leq \frac{4s}{n2^b}+  \sum_{j \not\in \mcal{J}_\alpha} \Pi_j^2,
\end{align}
where in the last inequality we use the fact that $\lba \hat{\mcal{J}}_\alpha \rba \leq s$ with probability $1$. Plugging \eqref{eq:Pi_bdd_1} and \eqref{eq:Pi_bdd_2} into \eqref{eq:var_bdd} yields
\begin{align}\label{eq:est_stage_bdd}
	\E\lb \sum_{j\in[d]} \lp \hat{\pi}_j - \pi_j \rp^2 \rb 
	& \leq \frac{8s}{n2^b}+4\Pr\lbp \mcal{J}_\alpha \not\subset \hat{\mcal{J}}_\alpha \rbp+\E\lb \sum_{j\not\in \mcal{J}_\alpha} \Pi^2_j \rb + 2\E\lb \sum_{j\in[d]}\lp \Pi_j - \pi_j \rp^2 \rb \nonumber\\
	& \overset{\text{(a)}}{\leq} \frac{8s}{n2^b}+4\Pr\lbp \mcal{J}_\alpha \not\subset \hat{\mcal{J}}_\alpha \rbp + \sum_{j\not\in \mcal{J}_\alpha} \lp \pi_j^2 + \Var\lp \Pi_j \rp\rp +\sum_{j \in [d]} \Var\lp \Pi_j \rp \nonumber\\
	& \overset{\text{(b)}}{\leq} \frac{8s}{n2^b}+4\Pr\lbp \mcal{J}_\alpha \not\subset \hat{\mcal{J}}_\alpha \rbp +s\alpha^2 + 2\sum_{j \in [d]} \Var\lp \Pi_j \rp\nonumber\\
	& \overset{\text{(c)}}{\leq} \frac{8s}{n2^b}+4\Pr\lbp \mcal{J}_\alpha \not\subset \hat{\mcal{J}}_\alpha \rbp +s\alpha^2 + \frac{1}{n},
\end{align}
where (a) holds since 1) $\Pi_j \sim \frac{2}{n}\msf{HG}\lp n, \frac{n}{2}, f_j\rp$ so $\E\lb \Pi_j \rb = \pi_j$, and 2) $\E\lb X^2 \rb = \lp \E X\rp^2 + \Var\lp X \rp$, (b) holds since by definition, for all $j \not\in \mcal{J}_\alpha$, $\pi_j \leq \alpha$, and (c) follows from the fact $\Var\lp \Pi_j \rp \leq \frac{\pi_j}{2n}$. Hence it remains to bound $\Pr\lbp \mcal{J}_\alpha \not\subset \hat{\mcal{J}}_\alpha \rbp$. 
								
Next, we prove that the localization schemes (Scheme~A, B, D) yields the same (oder-wise) bound of failure probability.
\paragraph{Scheme~A (uniform grouping)} We slightly modify the encoding scheme Algorithm~\ref{alg:enc_localization_1} such that each client $i$ in the localization stage (i.e. $i$ satisfies $\sigma(i)<n/2$) is assigned to a randomly selected group $\mcal{G}_m$ with $m \in \msf{uniform}(M)$ chosen by shared randomness, where $M = \frac{d}{2^b-1}$. Follow the same analysis as in Lemma~\ref{clm:1}, for any $j \in \mcal{J}_\alpha$, it can be successful localized if one of symbol $j$ in the first half sequence is assigned to $\mcal{G}_m$ with $j \in [(m-1)(2^b-1)+1: m(2^b-1)]$. Denote such event by $\mcal{E}_j$. Since there are $F_j' \eqDef f_j - F_j$ clients in the first stage observing $j$, the probability that symbol $j$ is not detected is 
\begin{equation}\label{eq:exp_F}
\Pr\lbp \mcal{E}^c_j \mv \sigma \rbp \leq \lp 1-\frac{1}{M}\rp^{F_j'} \leq e^{-\frac{F_j'}{M}} = e^{-\frac{F_j'\cdot (2^b-1)}{d}}.
\end{equation}
To further bound it, we apply Hoeffding's inequality on hyper-geometric distribution (notice that $F_j' \sim \msf{HG}\lp n, \frac{n}{2}, f_j \rp$).
\begin{lemma}[Hyper-geometric tail bound]\label{lemma:hg_bdd}
	Let $F \sim \msf{HG}(n, k, f)$ and define $p = k/n$. We have
	$$ \Pr\lbp F \leq (p-t)f \rbp \leq e^{-2t^2f},$$
	for all $t \in (0, p)$. 
\end{lemma}
Applying Lemma~\ref{lemma:hg_bdd} on \eqref{eq:exp_F} with $t = \frac{1}{4}$, we have
	$$ \E\lb \Pr\lbp \mcal{E}^c_j \mv \sigma \rbp\rb \leq \Pr\lbp F_j' > \frac{1}{4}f_j \rbp\cdot e^{-\frac{f_j\cdot (2^b-1)}{4d}} + \Pr\lbp F_j' \leq \frac{1}{4}f_j \rbp \leq e^{-\frac{f_j\cdot (2^b-1)}{4d}}+e^{-\frac{f_j}{2}}. $$
Therefore 
	$$ \Pr\lbp\bigcup_{j\in\mcal{J}_\alpha} \mcal{E}_j^c \rbp \leq \sum_{j\in\mcal{J}_\alpha} e^{-\frac{f_j\cdot (2^b-1)}{4d}}+e^{-\frac{f_j}{2}} \leq se^{-\frac{n\alpha (2^b-1)}{4d}}+se^{-\frac{n\alpha}{2}}. $$
pick $\alpha = \frac{1}{\sqrt{n2^b}}$ and by the same argument as in \eqref{eq:l2_final_bdd}, the $\ell_2$ estimation error can be bounded by $O\lp \max\lp\frac{s}{n2^b}, \frac{1}{n} \rp\rp$, as long as $n = \Omega\lp \frac{d^2\log d^2}{2^b} \rp$.
									
\paragraph{Scheme~B: random hashing} Let $\mcal{E}_\ell$ be defined as in the proof of Lemma~\ref{lemma:2}, and $j \in \mcal{J}_\alpha \setminus\mcal{C}_\ell$. Note that symbol $j \in \mcal{J}_\alpha$ can be detected if a client $i$ in localization stage observes $j$ and $W_i$ distinguishes $j$ under hypothesis $H_\ell$.  Since every channels are generated identically and independently, such probability can be controlled by \eqref{eq:distinguish_j}. Also notice that there are $F_j'$ clients in the first stage who observe $j$, the failure probability can be controlled by
	$$ \Pr\lbp \mcal{E}_{\ell} \mv \sigma \rbp \leq \lp 1-  \frac{2^b-1}{4s}\rp^{F_j'} \leq \exp\lp -F_j'\frac{2^b-1}{4s}\rp, $$
applying Lemma~\ref{lemma:hg_bdd} gives us 
	$$ \E\lb\Pr\lbp \mcal{E}_{\ell} \mv \sigma \rbp\rb \leq \exp\lp -nf_j\frac{2^b-1}{16s}\rp + e^{-\frac{f_j}{2}} \leq \exp\lp -n\alpha\frac{2^b-1}{32s}\rp + e^{-\frac{n\alpha}{2}}. $$
Taking union bound over $\ell \in [N]$, we obtain
	$$ \Pr\lbp \bigcup_{\ell\in[N]}\mcal{E}_\ell \rbp \leq \exp\lp -n \alpha\frac{2^b-1}{32s} + C_0s\log\lp \frac{d}{s}\rp\rp + e^{-\frac{n\alpha}{2}+ C_0s\log\lp \frac{d}{s}\rp}. $$
Selecting $\alpha = \frac{1}{\sqrt{n2^b}}$ and as in Section~\ref{sec:localization_hasing}, we conclude that as long as $n =\Omega\lp\frac{s^4\log^2\lp\frac{d}{s}\rp}{2^b}\rp$, the $\ell_2$ error \eqref{eq:hash_error_1} is $O\lp\max\lp\frac{s}{n 2^b}, \frac{1}{n}\rp\rp$.
Scheme~A and Scheme~B establishes the non-interactive part of Theorem~\ref{thm:freq_est}.
									
\paragraph{Scheme~D: tree-based recovery} As in Scheme~A, we replace every deterministic grouping with a randomized one. That is, in the localization step in Scheme~D, instead of set client $i$ with $\sigma(i) \equiv t (\,\msf{mod}\, \log d)$ into $\mcal{G}_t$, we assign it to group $t_i  \sim \msf{uniform}(\log d)$. Similarly, they will later be assigned to $\mcal{K}_{t, m}$ with $m\sim \msf{uniform}(M)$. Then conditioned on $\sigma$, all of the analysis in the proof of Lemma~\ref{lemma:localization_seq} holds, except that we replace \eqref{eq:interactive_err_bdd} with
\begin{align*}
	\Pr\lbp \mcal{E}^c_{t} \mv \mcal{E}_{t-1}, \sigma \rbp
	&= \Pr \lbp \mcal{J}_{\alpha, t} \not\subset \hat{\mcal{J}}_{\alpha, t}  \mv \mcal{J}_{\alpha, t-1} \subseteq \hat{\mcal{J}}_{\alpha, t-1}, \sigma\rbp \nonumber\\
	& \leq \Pr\lbp \exists j \in \mcal{B}_m \text{ s.t. } \forall i \in \mcal{K}_{m, t}\, X_i \neq j  \mv \sigma \rbp \\
	& \leq \lba \mcal{B}_m \rba \Pr\lbp \forall i \in \mcal{K}_{m, t}\, X_i \neq j  \mv \sigma \rbp \\
	& \overset{\text{(a)}}{\leq}\frac{2s}{2^b-1}\cdot \lp 1-\frac{1}{MT} \rp^{F'_j}\\
	&\leq \frac{2s}{2^b-1} \exp\lp \frac{F'_j (2^b-1)}{2s\log d} \rp.
\end{align*}
where (a) holds because we assign each client uniformly at random in $[T]\times[M]$. Applying Lemma~\ref{lemma:hg_bdd}, we have
										$$ \Pr\lbp \mcal{E}^c_{t} \mv \mcal{E}_{t-1}\rbp \leq \frac{2s}{2^b-1} \exp\lp \frac{ n\alpha(2^b-1)}{8s\log d} \rp + \exp\lp-\frac{n\alpha}{2}\rp.$$
Plugging into \eqref{eq:e_seq} and assume $\frac{8s}{2^b-1} > 2$, we obtain
										$$ \Pr\lbp  \mcal{E}^c_{T} \rbp \leq  T\cdot \frac{4s}{2^b-1}\cdot\exp\lp -\frac{n\alpha}{8s\log d} \rp = \exp\lp -\frac{n(2^b-1)\alpha}{8s\log d} +\lp \log\log d+\log\lp \frac{4s}{2^b-1} \rp\rp \rp.$$
										
Finally, picking $n_1 = \frac{n}{2}$ and $\alpha = \frac{1}{\sqrt{n2^b}}$ and combining Lemma~\ref{lemma:localization_seq} and Lemma~\ref{cor:two_stage_error}, we arrive at
\begin{align*}
	\E\lb \sum_{j\in[d]} \lp \hat{p}_j -p\rp^2 \rb 
	& \leq \exp\lp -\frac{n(2^b-1)\alpha}{8s\log d} +\lp \log\log d+\log\lp \frac{4s}{2^b-1} \rp\rp \rp+\frac{3s}{n2^b}+\frac{2}{n}\\
	& \leq C_0\cdot\lp\frac{s}{n 2^b}+ \frac{1}{n}\rp, 
\end{align*}
where the last inequality follows from Lemma~\ref{lemma:n_requirement} and the assumption 
$$n = \Omega\lp\frac{s^2\log^2 d}{2^b}\lp\log s + \log\log d \rp^2\rp.$$ 
This establishes the interactive part of Theorem~\ref{thm:freq_est}.
											
										
\end{proof}

\subsection*{Proof of Theorem~\ref{thm:almost_sparse}}
\begin{proof}
	Let $\mcal{S}$ be the set of symbols with $s$-highest probabilities, that is, $\mcal{S} = \lbp j\in[d]: p_j \geq p_{(s)} \rbp$\footnote{In case that $p_{(s)} = p_{(s+1)}$, we define $\mcal{S}$ to be an arbitrary set such that $\lba \mcal{S} \rba = s$ and $p_j \geq p_{(s)} \, \forall i \in \mcal{S}$.} and let $\beta \eqDef p_{(s+1)}$. As in previous section, write $\mcal{J}_\alpha \eqDef \lbp i: p_j \geq \max\lp\alpha, p_{(s)}\rp \rbp$\footnote{Again, with a slight abuse of notation, if $\alpha = p_{(s)} = p_{(s+1)}$, we require $\mcal{J}_\alpha = \mcal{S}$.} where $\alpha \geq 0$ will be specify later. Let $\hat{\mcal{J}}_\alpha$ be the output of the localization step in Scheme~B. 
	\begin{claim}\label{claim:1}
		$$\Pr\lbp\mcal{J}_\alpha  \not\subset \hat{\mcal{J}}_\alpha \rbp \leq n_1\frac{2^b-1}{4s}\lp 1-P_\mcal{S} \rp +  \exp\lp -n_1  \alpha\frac{2^b-1}{4s} + C_0 s\log d\rp.$$
		Moreover, $\lba \hat{\mcal{J}}_\alpha \rba < s$ almost surely.
	\end{claim}
	\begin{proof}
		Recall the in the decoding step of Scheme~B, we define all the candidate supports as $\mcal{C}_\ell$ for $\ell \in [N]$ (where $N \eqDef {d \choose s}$) and reduce the problem into multiple hypothesis $H_\ell: \mcal{J}_\alpha \subseteq \mcal{C}_\ell$. Now, without $s$-sparse assumption, we need a more detailed analysis and carefully bounding the failure probability.
		
		To begin with, define $\mcal{B} \subset [N]$ to be the indices of ``incorrect'' hypotheses, that is, $\mcal{B} \eqDef \lbp \ell \in [N]: \mcal{J}_\alpha \not\subset \mcal{C}_\ell \rbp$.  Then observe that the error event 
		$\mcal{E}_f \eqDef \lbp\mcal{J}_\alpha  \not\subset \hat{\mcal{J}}_\alpha \rbp$ occurs when 1) there exist one incorrect but consistent\footnote{Recall that the consistency is defined in Def~\ref{def:consistency}} candidates $\ell \in \mcal{B}$, i.e. for $\ell \in \mcal{B}$, define 
		$$ \mcal{F}_\ell \eqDef \lbp \mcal{C}_\ell \text{ is consistent with } \lp Y^{n_1}, W^{n_1} \rp \rbp; $$
		and 2) for all $\ell \not\in \mcal{B}$, $\mcal{C}_\ell$ does not consistent with $\lp Y^{n_1}, W^{n_1} \rp$, i.e. for all $\ell\not\in \mcal{B}$,
		$$\mcal{G}_\ell \eqDef \lbp \mcal{C}_\ell \text{ is \emph{not} consistent with } \lp Y^{n_1}, W^{n_1} \rp \rbp.$$
		Therefore, we can bound the failure event by
		$$\mcal{E}_f \subseteq  \lp \bigcup_{\ell \in \mcal{B}} \mcal{F}_\ell\rp \cup \lp \bigcap_{\ell\in [N]\setminus\mcal{B}} \mcal{G}_\ell \rp.$$
		
		\paragraph{Bounding $\mcal{F}_\ell$}
		We bound $\mcal{F}_\ell$ as in previous section: observe that $\mcal{F}_\ell$ cannot happen if there exist a client $i$ 1) who observes $j\in\mcal{J}_\alpha \setminus \mcal{C}_\ell$ and 2) whose channel $W_i$ distinguishes $j$. 
		Thus we have
		\begin{align}\label{eq:bound_F}
			\Pr\lbp \mcal{F}_\ell \rbp 
			& \leq \lp 1- \Pr\lbp \lbp X_i = j \rbp \cap \lbp W_i \text{ distinguishes } j \rbp \rbp\rp^{n_1} \nonumber\\
			& \leq \lp 1- \alpha \cdot \frac{2^b-1}{4s}\rp^{n_1} \leq \exp\lp -n_1  \alpha\frac{2^b-1}{4s}\rp.
		\end{align}
		
		\paragraph{Bounding $\mcal{G}_\ell$} Let $\ell^*$ be the index such that $\mcal{C}_{\ell^*} \eqDef \mcal{S}$. Then 
		\begin{equation}\label{eq:bound_G}
		\lp \bigcap_{\ell\in [N]\setminus\mcal{B}} \mcal{G}_\ell \rp \subseteq \mcal{G}_{\ell^*}.
		\end{equation}
		Observe that $\mcal{G}_{\ell^*}$ happens if there exist a client $i$ who 1) observes $j \neq \mcal{S}$ and 2) $W_i$ distinguishes $j$. Therefore 
		\begin{align}\label{eq:bound_G_2}
			\Pr\lbp \mcal{G}_{\ell^*} \rbp 
			& \leq \Pr\lbp \bigcup_{i\in[n]}\bigcup_{j\in [d]\setminus\mcal{S}} \lbp X_i = j \rbp \cap \lbp W_i \text{ distinguishes} j \rbp\rbp \nonumber\\
			& \overset{\text{(a)}}{\leq} n\sum_{j\in [d]\setminus\mcal{S}}\Pr\lbp \lbp X_i = j \rbp \cap \lbp W_i \text{ distinguishes} j \rbp\rbp\nonumber\\
			&  \overset{\text{(b)}}{\leq} n\frac{2^b-1}{4s}\lp\sum_{j\in [d]\setminus\mcal{S}}\Pr\lbp  X_i = j \rbp\rp \nonumber\\
			& = n\frac{2^b-1}{4s}\lp 1-P_\mcal{S} \rp,
		\end{align}
		where (a) is due to union bound, and (b) is because we generate $W_i$ independently with $X_i$.
		
		Finally, combining \eqref{eq:bound_F}, \eqref{eq:bound_G} and \eqref{eq:bound_G_2}, we obtain
		\begin{align*}
			\Pr\lbp \mcal{E}_f\rbp 
			&\leq \Pr\lbp \lp\bigcup_{\ell \in \mcal{B}} \mcal{F}_\ell\rp \cup \lp \bigcap_{\ell\in [N]\setminus\mcal{B}} \mcal{G}_\ell \rp\rbp \\
			& \leq n_1\frac{2^b-1}{4s}\lp 1-P_\mcal{S} \rp +  \exp\lp -n_1  \alpha\frac{2^b-1}{4s} + C_0 s\log d\rp,
		\end{align*}
		where the last inequality is due to union bound over $[N]  = {d \choose s}$.
	\end{proof}
	
	As in previous section, picking $\alpha = \frac{1}{\sqrt{n\cdot 2^b}}$ and by Claim~\ref{claim:1}, we can show that
	\begin{align*}
		\E\lb \sum_{j\in[d]} \lp \hat{p}_j -p\rp^2 \rb
		\leq \lp n_1\frac{2^b-1}{4s}+1\rp\lp 1-P_\mcal{S} \rp& + \exp\lp -\sqrt{\frac{n^2_1\lp 2^b -1 \rp^2}{n\cdot 2^b}}\frac{1}{4s}+C_0s\log\lp\frac{d}{s}\rp \rp\\
		& +\frac{2s}{(n-n_1)2^b}+\frac{1}{n-n_1}.
	\end{align*}
	Finally, picking $n_1 = C_1\cdot \sqrt{n\log n} \cdot s \cdot \log\lp \frac{d}{s} \rp$, we have
	$$ \E\lb \sum_{j\in[d]} \lp \hat{p}_j -p\rp^2 \rb  \leq C_2\cdot \lp \frac{s}{n\cdot2^b} + \frac{1}{n}\rp+C_3\cdot\sqrt{n\log n}\cdot 2^b\cdot\lp 1-P_\mcal{S} \rp. $$
\end{proof}




\end{document}